\documentclass{article}

\usepackage[nonatbib, final]{neurips_2021}

\usepackage[utf8]{inputenc} 
\usepackage[T1]{fontenc}    
\usepackage{hyperref}       
\usepackage{url}            
\usepackage{booktabs}       
\usepackage{amsfonts}       
\usepackage{nicefrac}       
\usepackage{microtype}      
\usepackage{xcolor}         

\usepackage{color}
\usepackage{amsmath,amssymb,amsthm,bm,bbm}
\usepackage{footmisc}

\usepackage{multirow}
\usepackage{multicol}
\usepackage{graphicx}

\usepackage{wrapfig}
\usepackage{lipsum}  

\usepackage{algorithm}
\usepackage{algorithmic}

\usepackage{subcaption}
\usepackage{graphicx}

\usepackage{epstopdf}
\usepackage{enumerate}

\usepackage{float}
\usepackage{tabularx}
\usepackage{makecell}
\usepackage{arydshln}

\usepackage{enumitem}

\usepackage{anyfontsize}
\usepackage{makecell}

\newtheorem{assumption}{Assumption}

\newcommand{\xv}{{\mathbf x}}

\newcommand{\Iv}{{\mathbf I}}

\newcommand{\Xc}{{\mathcal X}}

\newcommand{\Tc}{{\mathcal T}}

\newcommand{\loss}{{\mathcal L}}
\newcommand{\data}{{\mathcal D}}

\newcommand{\one}{{\mathbbm{1}}}

\newcommand{\Norm}{{\mathcal{N}}}

\newtheorem{proposition}{Proposition}
\newtheorem{lemma}{Lemma}

\newcommand{\proposed}{{SurvITE}}
\newcommand{\IPM}{\text{IPM}}

\newcommand{\minus}{{-}} 

\newcommand{\PR}{{\mathbb P}}

\newcommand{\tmax}{{t_{\text{max}}}}

\newcolumntype{L}[1]{>{\raggedright\let\newline\\\arraybackslash\hspace{0pt}}m{#1}}
\newcolumntype{C}[1]{>{\centering\let\newline\\\arraybackslash\hspace{0pt}}m{#1}}
\newcolumntype{R}[1]{>{\raggedleft\let\newline\\\arraybackslash\hspace{0pt}}m{#1}}

\newcommand\defeq{\mathrel{\stackrel{\makebox[0pt]{\mbox{\normalfont\tiny def}}}{=}}}

\newcommand{\indep}{\small {\raisebox{0.05em}{\rotatebox[origin=c]{90}{$\models$}}}}

\newsavebox\CBox 
\def\textBF#1{\sbox\CBox{#1}\resizebox{\wd\CBox}{\ht\CBox}{\textbf{#1}}}

\title{SurvITE: Learning Heterogeneous Treatment Effects from Time-to-Event Data}

\author{%
 Alicia Curth\thanks{Equal contribution} \\
  University of Cambridge\\
  \texttt{amc253@cam.ac.uk} \\
  \And
  Changhee Lee$^{*}$ \\
  Chung-Ang University \\
 \texttt{changheelee@cau.ac.kr} \\
  \And
    Mihaela van der Schaar \\
  University of Cambridge \\
  University of California, Los Angeles \\
  The Alan Turing Institute\\
  \texttt{mv472@cam.ac.uk}
}

\begin{document}
\maketitle
\begin{abstract} We study the problem of inferring heterogeneous treatment effects from time-to-event data. While both the related problems of (i) estimating treatment effects for \textit{binary or continuous} outcomes and (ii) \textit{predicting survival} outcomes have been well studied in the recent machine learning literature, their combination -- albeit of high practical relevance -- has received considerably less attention. With the ultimate goal of reliably estimating the effects of treatments on instantaneous risk and survival probabilities, we focus on the problem of learning (discrete-time) treatment-specific conditional hazard functions. We find that unique challenges arise in this context due to a variety of covariate shift issues that go beyond a mere combination of well-studied confounding and censoring biases. We theoretically analyse their effects by adapting recent generalization bounds from domain adaptation and treatment effect estimation to our setting and discuss implications for model design. We use the resulting insights to propose a novel deep learning method for treatment-specific hazard estimation based on balancing representations. We investigate performance across a range of experimental settings and empirically confirm that our method outperforms baselines by addressing covariate shifts from various sources.

\end{abstract}

\section{Introduction}
The demand for methods evaluating the effect of treatments, policies and interventions on \textit{individuals} is rising as interest moves from estimating population effects to understanding effect heterogeneity in fields ranging from economics to medicine. Motivated by this, the literature proposing machine learning (ML) methods for estimating the effects of treatments on continuous (or binary) end-points has grown rapidly, most prominently using tree-based methods \cite{hill2011bayesian, athey2016recursive, wager2018estimation, athey2019generalized, hahn2017bayesian}, Gaussian processes \cite{alaa2017bayesian, alaa2018limits}, and, in particular, neural networks (NNs) \cite{johansson2016learning, Shalit:16, johansson2018learning, shi2019adapting, hassanpour2019counterfactual, hassanpour2020learning, assaad2021counterfactual, curth2020}.  In comparison, the ML literature on heterogeneous treatment effect (HTE) estimation with time-to-event outcomes is rather sparse. This is despite the immense practical relevance of this problem -- e.g. many clinical studies consider time-to-event outcomes; this could be the time to onset or progression of disease, the time to occurrence of an adverse event such as a stroke or heart attack, or the time until death of a patient.

In part, the scarcity of HTE methods may be due to time-to-event outcomes being inherently more challenging to model, which is attributable to two factors \cite{tutz2016modeling}: (i) time-to-event outcomes differ from standard regression targets as the main objects of interest are usually not only expected survival times but the \textit{dynamics of the underlying stochastic process}, captured by hazard and survival functions, and (ii) the presence of \textit{censoring}. This has led to the development of a rich literature on survival analysis particularly in (bio)statistics, see e.g. \cite{tutz2016modeling, klein2003survival}. Classically, the effects of treatments in clinical studies with time-to-event outcomes are assessed by examining the coefficient of a treatment indicator in a (semi-)parametric model, e.g. Cox proportional hazards model \cite{Cox:72}, which relies on the often unrealistic assumption that models are correctly specified. Instead, we therefore adopt the nonparametric viewpoint of van der Laan and colleagues \cite{van2011targeted, stitelman2010collaborative, stitelman2011targeted, cai2019targeted} who have developed tools to incorporate ML methods into the estimation of treatment-specific \textit{population average} parameters. Nonparametrically investigating treatment effect \textit{heterogeneity}, however, has been studied in much less detail in the survival context. While a number of tree-based methods have been proposed recently \cite{tabib2020non, henderson2020individualized, zhang2017mining, cui2020estimating}, NN-based methods lack extensions to the time-to-event setting despite their successful adoption for estimating the effects of treatments on other outcomes --  the only exception being \cite{chapfuwa2021enabling}, who directly model event times under different treatments with generative models. 

Instead of modeling event times directly like in \cite{chapfuwa2021enabling}, we consider adapting machine learning methods, with special focus on NNs, for estimation of (discrete-time) treatment-specific hazard functions. We do so because many target parameters of interest in studies with time-to-event outcomes are functions of the underlying temporal dynamics; that is, hazard functions can be used to directly compute (differences in) survival functions, (restricted) mean survival time, and hazard ratios. We begin by exploring and characterising the unique features of the survival treatment effect problem within the context of empirical risk minimization (ERM); to the best of our knowledge, such an investigation is lacking in previous work. In particular, we show that learning treatment-specific hazard functions is a challenging problem due to the potential presence of \textit{multiple} sources of \textit{covariate shift}: (i) non-randomized treatment assignment (confounding), (ii) informative censoring and (iii) a form of shift we term \textit{event-induced} covariate shift, all of which can impact the quality of hazard function estimates.  We then theoretically analyze the effects of said shifts on ERM, and use our insights to propose a new NN-based model for treatment effect estimation in the survival context.

\textbf{Contributions} (i) We identify and formalize key challenges of  heterogeneous treatment effect estimation in time-to-event data within the framework of  ERM. In particular, as discussed above, we show that when estimating treatment-specific hazard functions, \textit{multiple} sources of covariate shift arise. (ii) We theoretically analyse their effects by adapting recent generalization bounds from domain adaptation and treatment effect estimation to our setting and discuss implications for model design. This analysis provides new insights that are of independent interest also in the context of hazard function estimation in the absence of treatments. (iii) Based on these insights, we propose a new model (\proposed) relying on balanced representations that allows for estimation of treatment-specific target parameters (hazard and survival functions) in the survival context, as well as a sister model (SurvIHE), which can be used for individualized hazard estimation in standard survival settings (without treatments). We investigate performance across a range of experimental settings and empirically confirm that \proposed~outperforms a range of natural baselines by addressing covariate shifts from various sources.

\section{Problem Definition} \label{sec:problem_definition}
In this section, we discuss the problem setup of heterogeneous treatment effect estimation from time-to-event data, our target parameters and assumptions. In Appendix A, we present a self-contained introduction to and comparison with heterogeneous treatment effect estimation with standard (binary/continuous) outcomes.

\textbf{Problem setup.} Assume we observe a time-to-event dataset $\data = \{ (a_{i}, x_{i}, \tilde{\tau}_{i}, \delta_{i}) \}_{i=1}^{n}$ comprising realizations of the tuple $(A, X, \tilde{T}, \Delta) \sim \mathbb{P}$ for $n$ patients. 
Here, $X \in \Xc$ and $A \in \{0,1\}$ are random variables for a covariate vector describing patient characteristics and an indicator whether a binary treatment was administered at baseline, respectively.
Let $T \in \Tc$ and $C \in \Tc$ denote random variables for the time-to-event and the time-to-censoring; here, events are usually \textit{adverse}, e.g. progression/onset of disease or even death, and censoring indicates loss of follow-up for a patient. Then, the \textit{observed} time-to-event outcomes of each patient are described by $\tilde{T} = \min (T, C)$ and $\Delta = \one(T \leq C)$, which indicate the time elapsed until either an event or censoring occurs and whether the event was observed or not, respectively. 
Throughout, we treat survival time as discrete\footnote{Where necessary, discretization can be performed by transforming continuous-valued times into a set of contiguous time intervals, i.e., $T = \tau$ implies $T \in [t_{\tau}, t_{\tau} + \delta t)$ where $\delta t$ implies the temporal resolution.} and the time horizon as finite with pre-defined maximum $\tmax$, so that the set of possible survival times is $\Tc = \{1, \cdots, \tmax\}$. 

We transform the \textit{short} data structure outlined above to a so-called \textit{long} data structure which can be used to \textit{directly} estimate conditional hazard functions using standard machine learning methods \cite{stitelman2010collaborative}. We define two counting processes $N_T(t)$ and $N_C(t)$ which track events and censoring, i.e. $N_T(t)=\mathbbm{1}(\tilde{T} \leq t, \Delta =1)$ and $N_C(t)=\mathbbm{1}(\tilde{T} \leq t, \Delta =0)$ for $t \in \Tc$; both are zero until either an event or censoring occurs. By convention, we let $N_T(0)=N_C(0)=0$. Further, let $Y(t) = \one(N_{T}(t)=1 \cap N_{T}(t\minus 1)=0)$ be the indicator for an event occuring exactly at time $t$; thus, for an individual with $\tilde{T}=\tau$ and $\Delta=1$, $Y(t) = 0$ for all $t\neq \tau$, and $Y(t)=1$ at the event time $t=\tau$. The conditional hazard is the probability that an event occurs \textit{at} time $\tau$ given that it does not occur before time $\tau$, hence it can be defined as \cite{cai2019targeted}
\begin{equation}\label{eq:haz}
    \begin{split}
      \lambda(\tau|a,x) &= \PR(\tilde{T}=\tau, \Delta = 1|\tilde{T} \geq \tau, A=a, X=x) \\
      &= \PR(Y(\tau)=1| N_T(\tau \minus 1)=_C(\tau \minus 1)=0, A=a, X=x) 
      \end{split}
\end{equation}
It is easy to see from (\ref{eq:haz}) that given data in long format, $\lambda(\tau|a,x)$ can be estimated for any $\tau$ by solving a standard classification problem with $Y(\tau)$ as target variable, considering only the samples \textit{at-risk} at time $\tau$ in each treatment arm (individuals for which neither event nor censoring has occurred until that time point; i.e. the set $\mathcal{I}(\tau, a) \defeq \{i \in [n]: N_{T}(\tau \minus 1)_{i} = N_{C}(\tau \minus 1)_{i} =0 \cap  A_i=a\}$). Finally, given the hazard, the associated survival function $S(\tau|a,x) = \PR(T>\tau |A=a, X=x)$ can then be computed as $S(\tau|a,x)=\prod_{t \leq \tau}\big(1 - \lambda(t|a, x)\big)$. The censoring hazard $\lambda_{C}(t|a,x)$ and survival function $S_C(t|a, x)$ can be defined analogously. 

\textbf{Target parameters.} 
While the main interest in the standard treatment effect estimation setup with continuous outcomes usually lies in estimating only the (difference between) conditional outcome means under different treatments, there is a broader range of target parameters of interest in the time-to-event context, including both treatment-specific target functions and \textit{contrasts} that represent some form of heterogeneous treatment effect (HTE). 
We define the treatment-specific (conditional) hazard and survival functions as 
\begin{equation}
    \begin{split}
    \lambda^{a}(\tau|x) &= \PR(T=\tau | T \geq \tau, do(A=a, C \geq \tau), X=x)\\
    S^{a}(\tau|x) &= \PR(T>\tau| do(A=a, C \geq \tau), X=x) = \prod\nolimits_{t \leq \tau}\big(1 - \lambda^{a}(t| x)\big)
    \end{split}
\end{equation}
Here, $do(\cdot)$ denotes \cite{Pearl:09}'s do-operator which indicates an intervention; in our context, $do(A=a, C \geq \tau)$ ensures that every individual is assigned treatment $a$ \textit{and} is observed at (not censored before) the time-step of interest \cite{stitelman2010collaborative}. Below we discuss assumptions that are necessary to identify such interventional quantities from observational datasets in the presence of censoring. \\
Given $\lambda^{a}(\tau|x)$ and $S^{a}(\tau|x)$, possible HTEs of interest\footnote{\textit{Note:} All parameters of interest to us are \textit{heterogeneous} (also sometimes referred to as \textit{individualized}), i.e. a function of the covariates $X$, while the majority of existing literature in (bio)statistics considers \textit{population average} parameters that are functions of quantities such as $\PR(T>\tau| do(A=a))$, which average over all $X$.} include the difference in treatment-specific survival times at time $\tau$, i.e.
 $\text{HTE}_{surv}(\tau|x) = S^1(\tau|x) - S^0(\tau|x)$, the difference in restricted mean survival time (RMST) up to time $L$, i.e. $\text{HTE}_{rmst}(x)  = \sum_{t_{k} \leq L} \big( S^1(t_{k}|x) - S^0(t_{k}|x) \big) \cdot (t_{k}- t_{k-1})$,  and hazard ratios. In the following, we will focus on estimation of the treatment specific hazard functions $\{\lambda^{a}(t|x)\}_{a\in \{0,1\}, t\in \mathcal{T}}$ as this can be used to compute survival functions and causal contrasts. 

\begin{wrapfigure}{r}{0.52\textwidth}
\vskip -0.22in
    \centering
    \includegraphics[width=0.5\textwidth]{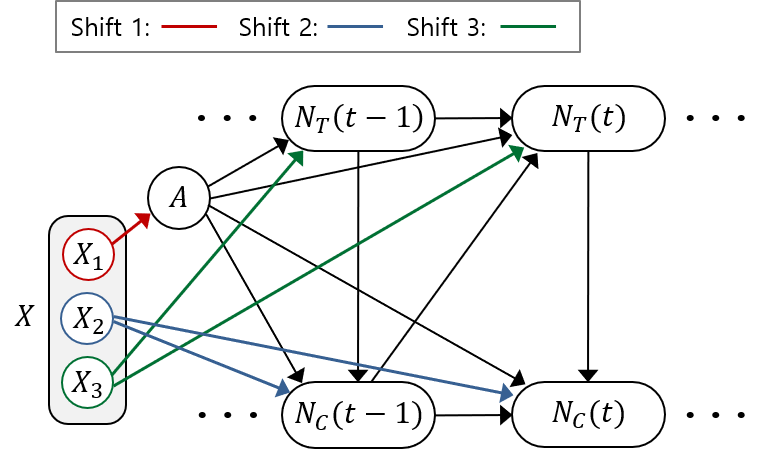}
    \caption{The assumed underlying DAG. Covariates $X$ can be split into (possibly overlapping) subsets $X_1$, $X_2$ and $X_3$, determining treatment selection, informative censoring, and event times, respectively. } \vspace{-5mm}
 \label{fig:DAG}
\end{wrapfigure}
\textbf{Assumptions.} (\textit{1. Identification}) To identify interventional quantities from observational data, it is necessary to make a number of \textit{untestable} assumptions on the underlying data-generating process (DGP) \cite{Pearl:09} -- this generally limits the ability to make causal claims to settings where sufficient domain knowledge is available. Here, as \cite{stitelman2010collaborative, stitelman2011targeted}, we assume the data was generated from the fairly general directed acyclic graph (DAG) presented in Fig. \ref{fig:DAG}. As there are no arrows originating in hidden nodes entering treatment or censoring nodes, this graph formalizes \textit{(1.a) The `No Hidden Confounders' Assumption} in static treatment effect estimation and \textit{(1.b) The `Censoring At Random' Assumption} in survival analysis \cite{stitelman2010collaborative}. The latter is necessary here as estimating an effect of treatment on event time implicitly requires that censoring can be `switched off' -- i.e. intervened on. This graph implicitly also formalizes \textit{(1.c) The Consistency Assumption}, i.e. that observed outcomes are `potential' outcomes under the observed intervention, as each node in a DAG is \textit{defined} as a function of its ancestors and exogenous noise \cite{stitelman2010collaborative}. Under these assumptions, $\lambda^{a}(\tau|x) = \lambda(\tau|a,x)$.

(\textit{2. Estimation}) To enable nonparametric estimation of $\lambda^{a}(\tau|x)$ for some $\tau \in \Tc$, we additionally need to assume that the interventions of interest are observed with non-zero probability; within different literatures these assumptions are known under the label of `overlap' or `positivity' \cite{Shalit:16, van2011targeted}. In particular, for $0 < \epsilon_1, \epsilon_2,  < 1$ we need that \textit{(2.a)}  $\epsilon_1 < \PR(A=a|X=x)<1-\epsilon_1$, i.e. treatment assignment is non-deterministic, and that \textit{(2.b)} $\PR(N_C(t)=0|A=a, X=x)>\epsilon_2$ for all $t < \tau$, i.e. no individual will be deterministically censored before $\tau$. Finally, because $\lambda^{a}(\tau|x)$ is a probability defined \textit{conditional} on survival up to time $\tau$, we need to assume that \textit{(2.c)} $\mathbb{P}(N_T(\tau \minus 1)=0|A=a, X=x)>\epsilon_3>0$ for it to be well-defined. We formally state and discuss all assumptions in more detail in Appendix C.

\section{Challenges in Learning Treatment-Specific Hazard Functions using ERM} 
\paragraph{Preliminaries: ERM under Covariate Shift.} Recall that in problems with covariate shift, the training distribution $X, Y \sim  \mathbb{Q}_0(\cdot)$ used for ERM and target distribution $X, Y \sim  \mathbb{Q}_1(\cdot)$ are mismatched: One assumes that the marginals do not match, i.e. $\mathbb{Q}_0(X) \neq \mathbb{Q}_1(X)$, while the conditionals remain the same, i.e. $\mathbb{Q}_0(Y|X)=\mathbb{Q}_1(Y|X)$ \cite{kouw2018introduction}. If the hypothesis class $\mathcal{H}$ used in ERM does not contain the truth (or in the presence of heavy regularization), this can lead to suboptimal hypothesis choice as $\arg\min_{h \in \mathcal{H}} \mathbb{E}_{X, Y \sim \mathbb{Q}_1(\cdot)}[\ell(Y, h(X))] \neq \arg \min_{h \in \mathcal{H}} \mathbb{E}_{X, Y \sim \mathbb{Q}_0(\cdot)}[\ell(Y, h(X))]$ in general.

\subsection{Sources of Covariate Shift in Learning Treatment-Specific Hazard Functions}
We now consider how to learn a treatment-specific hazard function $\lambda^{a}(\tau|x)$ from observational data using ERM. As detailed in Section \ref{sec:problem_definition},  we exploit the long data format by realizing that $\lambda^{a}(\tau|x)$ can be estimated by solving a standard classification problem with $Y(\tau)$ as dependent variable and $X$ as covariates, using only the samples at risk with treatment status $a$, i.e. $\mathcal{I}(\tau, a)$, which corresponds to solving the empirical analogue of the problem 
\begin{equation}
    \hat{\lambda}^{a}(\tau|x)  \in  \arg \min_{h_{a, \tau} \in \mathcal{H}} \mathbb{E}_{X, Y(\tau) \sim \mathbb{P}_{a, \tau}(\cdot)}[\ell(Y(\tau), h_{a, \tau}(X)]
\end{equation}
where we use $\mathbb{P}_{a, \tau}$ to refer to the observational (at-risk) distribution $\mathbb{P}_{a, \tau}(X, Y(\tau)) = {\lambda}^{a}_T(\tau|X)\mathbb{P}_{a,\tau}(X)$ with $\mathbb{P}_{a,\tau}(X) = \mathbb{P}(X|N_T(\tau\minus1)=N_C(\tau\minus1)= 0, A = a) =  \mathbb{P}(X|\tilde{T}\geq \tau, A= a)$. If the loss function $\ell$ is chosen to be the log-loss, this corresponds to optimizing the likelihood of the hazard. 

The observational (at-risk) covariate distribution $\mathbb{P}_{a,\tau}(X)$, however, is \textit{not} our target distribution:  instead, to obtain reliable hazard estimates for the whole population, we wish to optimize the fit over the population at baseline, i.e. the marginal distribution $X \sim \mathbb{P}(X)$ which we will refer to as  $\mathbb{P}_0(X)$ below to emphasize it being the baseline at-risk distribution\footnote{With slight abuse of notation, we will use $\mathbb{P}_0$ and $\mathbb{P}_{a,\tau}$ also to refer to densities of continuous $x$}. Here, differences between $\mathbb{P}_0(X)$ and the population at-risk $\mathbb{P}_{a,\tau}(X)$ can arise due to three distinct sources of covariate shift:
\begin{itemize}[leftmargin=5.5mm]
    \item \textit{(Shift 1) Confounding/treatment selection bias}: if treatment is not assigned completely at random, then $\mathbb{P}(X|A=a)\neq \mathbb{P}_0(X)$ and the distribution of characteristics across the treatment arms differs already at baseline, thus $\mathbb{P}_{a,\tau}(X)\neq \mathbb{P}_0(X)$ in general.
    \item \textit{(Shift 2) Censoring bias}: regardless of the presence of confounding, if the censoring hazard is not independent of covariates, i.e. $\lambda_C(\tau|a,x) \neq \lambda_C(\tau|a)$, then the population at-risk changes over time such that $\mathbb{P}_{a, \tau_1}(X) \neq \mathbb{P}_{a, \tau_2}(X) \neq\mathbb{P}_0(X)$ in general. If, in addition, there are differences between the treatment-specific censoring hazards, then the at-risk distribution will also differ across treatment arms at any given time-point, i.e. $\mathbb{P}_{a,\tau}(X) \neq \mathbb{P}_{1-a, \tau}(X)$ for $\tau>1$ in general. 
    \item \textit{(Shift 3) Event-induced shifts}: Counterintuitively, even in the absence of both confounding and censoring, there will be covariate shift in the at-risk population if the event-hazard depends on covariates, i.e. if $\lambda(\tau|a,x) \neq \lambda(\tau|a)$ then $\mathbb{P}_{a, \tau_1}(X) \neq \mathbb{P}_{a,\tau_2}(X) \neq\mathbb{P}_0(X)$ in general. Further, if there are heterogenous treatment effects, then $\mathbb{P}_{a,\tau}(X) \neq \mathbb{P}_{1\minus a, \tau}(X)$ for $\tau>1$ in general.
\end{itemize}

\textbf{What makes the survival treatment effect estimation problem unique?} While \textit{Shift 1} arises also in the standard treatment effect estimation setting, \textit{Shift 2} and \textit{Shift 3} arise uniquely due to the nature of time-to-event data\footnote{Interestingly, changes of the at-risk population over time arise also in standard survival problems (without treatments); yet in the context of \textit{prediction} these do not matter: as the at-risk population at any time-step is also the population that will be encountered at test-time, this shift in population over time is not problematic, unless it is caused by censoring. If, however, our goal is \textit{estimation} of the best target parameter (here: the hazard at a specific point in time $\tau$) over the whole population, this corresponds to a setting where the ideal evaluation is performed on a population different from the observed one  -- and hence requires careful consideration of the consequences of the covariate shifts discussed above.}. Thus, estimating treatment effects from time-to-event data is inherently more involved than estimating treatment effects in the standard static setup, as covariate shift at time horizon $\tau>1$ can arise \textit{even in a randomized control trial (RCT)}. Thus, in addition to the overall at-risk population changing over time, both treatment effect heterogeneity and treatment-dependent censoring can lead to differences in the composition of the population at-risk in each treatment arm. Further, Shifts 1, 2 and 3 can also interact to create more extreme shifts; e.g. if treatment selection is based on the same covariates as the event process (i.e. $X_1=X_3$ in Fig. \ref{fig:DAG}) then event-induced shift can amplify the selection effect over time (refer to Appendix E for a synthetic example of this).

\subsection{Possible Remedies and Theoretical Analysis}
A natural solution to tackle bias in ERM caused by covariate shift is to use importance weighting \cite{shimodaira2000improving}; i.e. to reweight the empirical risk by the density ratio of target $\mathbb{P}_0(X)$ and observed distribution $\mathbb{P}_{a,\tau}(X)$. If we wanted to obtain a hazard estimator for $(\tau, a)$, optimized towards the marginal population,  optimal importance weights are given by
$w^*_{a, \tau}(x) = \frac{\mathbb{P}_0(x)}{\mathbb{P}_{a,\tau}(x)} = \frac{p_{\tau, a}}{e_a(x)r_a(x, \tau)}$ with  $p_{\tau, a}=\mathbb{P}(\tilde{T}\geq \tau, A=a)$, $e_a(x)=\mathbb{P}(A=a|X=x)$ the propensity score, and $r^{a}(x, \tau) = \PR(\tilde{T} \geq \tau | A=a, X=x)$ the probability to be at risk, i.e. the probability that neither event nor censoring occurred before time $\tau$. These weights are well-defined due to the overlap assumptions detailed in Sec. \ref{sec:problem_definition}; however, they are in general unknown as they \textit{depend on the unknown target parameters} $\lambda^{a}(\tau|x)$ through $r^{a}(x, \tau)$. Further, especially for large $\tau$, these weights might be very extreme even if known, which can lead to highly unstable results \cite{cortes2010learning} -- making biased yet stabilized weighting schemes, e.g. truncation, a good alternative. Therefore, we only assume access to some (possibly imperfect) weights $w_{a, \tau}(x)$ s.t. $\mathbb{E}_{X \sim \mathbb{P}_{a, \tau}}[w_{a, \tau}(x)]=1$, so that we can create a weighted distribution $\mathbb{P}^w_{a, \tau}=w_{a, \tau}(x)\mathbb{P}^a_\tau(x)$. (Note: $\mathbb{P}^a_\tau(x)$ can be recovered by using $w_{a, \tau}(x)=1$.)

Either instead of \cite{johansson2016learning, Shalit:16} or in addition to weighting \cite{johansson2018learning, hassanpour2019counterfactual, assaad2021counterfactual, johansson2020generalization}, the literature on learning balanced representations for static treatment effect estimation has focused on finding a different remedy for distributional differences between treatment arms: creating representations $\Phi: \mathcal{X} \rightarrow \mathcal{R}$ which have similar (weighted) distributions across arms as measured by an integral probability metric (IPM), motivated by generalization bounds. As we show below, we can exploit a similar feature in our context by finding a representation that minimizes the IPM term not between treatment arms, but between covariate distribution at baseline $\mathbb{P}_0$ and $\mathbb{P}^w_{a, \tau}$. The proposition below bounds the target risk of a hazard estimator $\textstyle{\hat{\lambda}^a(\tau|x)=h(\Phi(x))}$ relying on any representation. The proof, which relies on the concept of excess target information loss, proposed recently to analyze domain-adversarial training \cite{johansson2019support}, and the standard IPM arguments made in e.g. \cite{johansson2020generalization}, is stated in Appendix C.

\begin{proposition} For fixed $a, \tau$ and representation $\Phi: \mathcal{X} \rightarrow \mathcal{R}$, let $\mathbb{P}_0^{\Phi}$, $\mathbb{P}^{\Phi}_{a, \tau}$ and $\mathbb{P}^{w, \Phi}_{a, \tau}$ denote the target, observational, and weighted observational distribution of the representation $\Phi$. Define the pointwise losses
\begin{equation}
\begin{split}
    \ell_{h, \mathbb{Q}}(x; a, \tau) \defeq \mathbb{E}_{Y(\tau)|x, a \sim \mathbb{Q}}[\ell(Y(\tau), h(\Phi(X)))|X=x, A=a] \\
    \ell_{h, \mathbb{Q}^\Phi}(\phi; a, \tau) \defeq \mathbb{E}_{Y(\tau)|\phi, a \sim \mathbb{Q}^\Phi}[\ell(Y(\tau), h(\Phi))|\Phi=\phi, A=a]
\end{split}
\end{equation}
of (hazard) hypothesis $h \equiv h_{a, \tau}: \mathcal{R} \rightarrow [0,1]$ w.r.t. distributions in covariate and representation space, respectively. Assume there exists a constant $C_\Phi>0$ s.t. ${C_\Phi}^{-1} \ell_{h,\mathbb{P}^{w, \Phi}_{a, \tau}}(\phi, a, \tau)
 \in \mathcal{G}$ for some family of functions $\mathcal{G}$. Then we have that 
\begin{equation}\label{eq:mainbound}
   \underbrace{\mathbb{E}_{X \sim \mathbb{P}_0}[\ell_{h, \mathbb{P}}(X; a, \tau)]}_{\text{Target Risk}} \leq  \underbrace{\mathbb{E}_{X \sim \mathbb{P}_{a, \tau}}[w_{a, \tau}(X)\ell_{h, \mathbb{P}}(X; a, \tau)]}_{\text{Weighted observational risk}} + C_\Phi \underbrace{\IPM_G(\mathbb{P}_0^{\Phi}, \mathbb{P}^{w, \Phi}_{a, \tau})}_{\text{Distance in } \Phi \text{-space}} + \underbrace{\eta^l_\Phi(h)}_{\text{Info loss}}
\end{equation}
where $\IPM_\mathcal{G}(\mathbb{P}, \mathbb{Q}) = \sup_{g\in \mathcal{G}}\left|\int g(x) (\mathbb{P}(x) - \mathbb{Q}(x))dx\right| $ and we define the excess target information loss $\eta^{\ell}_\Phi(h)$ analogously to \cite{johansson2019support} as 
    $\eta^{\ell}_\Phi(h) \defeq \mathbb{E}_{X\sim \mathbb{P}}[\xi_{\mathbb{P}_0^\Phi, \mathbb{P}}(X) - \xi_{\mathbb{P}^{w, \Phi}_{a, \tau}, \mathbb{P}}(X)]$  with $\xi_{ \mathbb{Q}^\Phi, \mathbb{Q}}(x) \defeq  \ell_{h, \mathbb{Q}^\Phi}(\phi; a, \tau) - \ell_{h, \mathbb{Q}}(x; a, \tau)$. For invertible $\Phi$, $\eta^{\ell}_\Phi(h)=\xi_{\mathbb{Q}^\Phi, \mathbb{Q}}(x)=0$.
\end{proposition} 

Unlike the bounds provided in \cite{Shalit:16, johansson2018learning, johansson2020generalization, assaad2021counterfactual, chapfuwa2021enabling}, this bound does not rely on representations to be invertible; we consider this feature important as none of the works listed actually enforced invertibility in their proposed algorithms. Given bound (\ref{eq:mainbound}), it is easy to see why non-invertibilty can be useful: for any (possibly non-invertible) representation for which it holds that $Y(\tau) \indep X | \Phi(X), A$, it also holds that $\eta^\ell_\Phi(h)=\xi_{\mathbb{P}^\Phi, \mathbb{P}}(x)=\xi_{\mathbb{P}^{w,\Phi}_{a, \tau}, \mathbb{P}}(x)=0$ and the causally identifying restrictions continue to hold. A simple representation for which this property holds is a selection mechanism that chooses only the causal parents of $Y(\tau)$ from within $X$; if $X$ can be partitioned into variables affecting the instantaneous risk ($X_3$ in Fig. \ref{fig:DAG}), and variables affecting \textit{only} treatment assignment ($X_1 \setminus X_3$) and/or censoring mechanism  ($X_2 \setminus X_3$), then the IPM term can be reduced by a representation which drops the latter two sets of variables --  or irrelevant variables correlated with any such variables -- without affecting $\eta^\ell_\Phi(h)$. As a consequence, event-induced covariate shift can generally not be \textit{fully} corrected for using non-invertible representations (unless the variables affecting event time are different at every time-step). Further, given perfect importance weights $w^*$, both $\eta^{\ell}_\Phi(h)$ and IPM term are zero.

Except for the dependence on $\eta^\ell_\Phi(h)$, this bound differs from the regression-based bound for survival treatment effects stated in \cite{chapfuwa2021enabling} (which is identical to the original treatment effect bound in \cite{Shalit:16}) in that we have dependence on $\tau$ in the IPM term, which, among other things, explicitly captures the effect of censoring. Our bound motivates that, instead of finding representations that balance treatment- and control group at baseline (or at each time step) we should find representations that balance $\mathbb{P}^\Phi_{a, \tau}$ towards the \textit{baseline distribution} $\mathbb{P}_0^\Phi$ for each time step, which motivates our method detailed below. If, instead, we would apply the IPM-term to encourage only the arm-specific at-risk distributions at each time-step  to be similar, this would correct only for shifts due to (i) confounding at baseline, (ii) treatment-induced differences in censoring and (iii) treatment-induced differences in events. It will, however, not allow handling the event- and censoring-induced shifts that occur regardless of treatment status.  Note that this bound therefore also motivates the use of balanced representations for modeling time-to-event outcomes in the presence of informative censoring even in the standard prediction setting, which is a finding that could be of independent interest for the ML survival analysis literature.

\subsection{From hazards to survival functions}
If the ultimate goal is to use the hazard function to estimate survival functions as $ \hat{S}^{a}(\tau|x) = \prod\nolimits_{t \leq \tau}\big(1 - \hat{\lambda}^{a}(t| x)\big)$, the best target population to consider during hazard estimation may not be the marginal distribution. Instead, the optimal target population may depend on the metric by which we wish to evaluate the resulting survival function. If we wanted to find the survival function that maximises the complete data likelihood (corresponding to the hypothetical setting in which we intervened to set $A=a$ and $C \geq \tau$), the target population (at each time step $t$) would be $\mathbb{P}(X|T \geq t, do(A=a, C \geq t))$ -- the population that preserves event-induced shift but removes selection- and censoring-induced shifts. If, instead, we focused on the MSE of estimating the survival function (as in our experiments), it becomes more difficult to derive an exact target population for estimating the hazards. If we assume access to a perfect estimate of the survival function for the first $\tau\minus 1$ time steps (i.e. $\hat{S}(\tau\minus 1|x)=S(\tau\minus 1|x)$) and focus only on estimating the next hazard, ${\lambda}^a(\tau|X)$, we can write 
\begin{equation*}
\begin{split} 
\mathbb{E}_{X \! \sim \mathbb{P}_0}[(S^a\!(\tau|X)\!-\!\hat{S}^a\!(\tau|X))^2] &= \mathbb{E}_{X \!\sim \mathbb{P}_0}[(S^a\!(\tau \minus 1|X)(1\!-\!\lambda^a\!(\tau|X))\!-\!\hat{S}^a\!(\tau \minus 1|X)(1\!-\!\hat{\lambda}^a\!(\tau|X)))^2]\\ & =\mathbb{E}_{X \! \sim \mathbb{P}_0}[S^a\!(\tau\minus 1|X)^2(\hat{\lambda}^a\!(\tau|X)\!-\!{\lambda}^a\!(\tau|X))^2]
\end{split}
\end{equation*}
and notice that the MSE will implicitly down-weigh individuals with lower survival probability\footnote{Due to the square in the term $S^a(\tau-1|X)^2$, this will be even more extreme than exact up-weighting of the population $\mathbb{P}(X|T \geq \tau, do(A=a, C \geq \tau))$.}. 

Defining an exact target population for the hazard when the goal is to also estimate the survival function well is thus not straightforward, making exact importance weighting difficult. Additionally, unlike the marginal population, interventional populations which change over time, such as  $\mathbb{P}(X|T \geq t, do(A=a, C \geq t))$, are never observed in practice and hence cannot be used to perform balancing regularization of a representation using empirical IPM-terms. Therefore,  we refrain from using importance weighting in our method (which is described in the next section), and resort to using the marginal population for balancing regularization of the representation throughout. Intuitively, as outlined in the previous section, we expect that doing so will \textit{not} over-correct for event-induced shifts that are predictive of outcome (and should hence be preserved) as such ``over-balancing" would reduce the predictive power of the representation, which would immediately be penalized by the presence of the expected loss component in the bound\footnote{In practice, we ensure this by weighting the contribution of the IPM term by a hyperparameter that is chosen to preserve predictive performance of the representation (see Appendix D).}. Additionally, we expect that using the marginal population for balancing could be useful also for estimating the survival function even in the absence of selection- and censoring-induced shifts, as it may help to remove the effect of variables that appear spuriously correlated with outcome over time.

\section{\proposed: Estimating  HTEs from Time-to-Event Data} 

\begin{figure}[!t]
    \centering
    \includegraphics[width=0.6\textwidth]{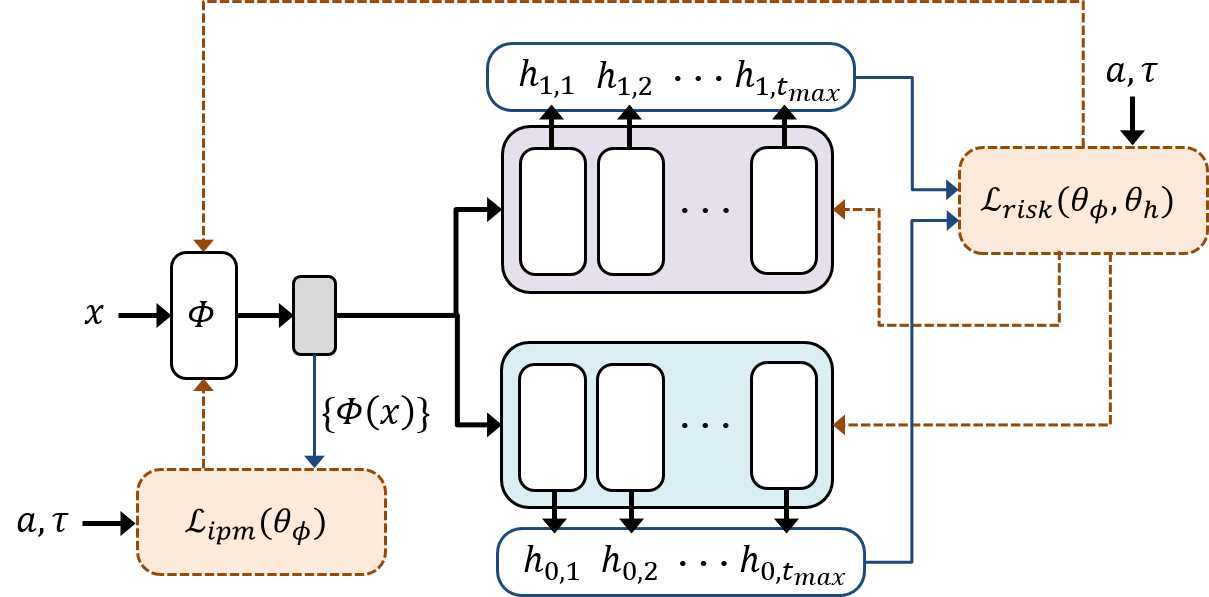}
    \caption{\proposed~architecture.} \vspace{-3mm}
    \label{fig:network_architecture}
\end{figure}
Based on the theoretical analysis above, we propose a novel deep learning approach to HTE estimation from observed time-to-event data, which we refer to as \proposed~{(Individualized Treatment Effect estimator for Survival analysis)}.\footnote{The source code for \proposed~is available in {\url{https://github.com/chl8856/survITE}}.} The network architecture is illustrated in Figure \ref{fig:network_architecture}. Note that even in the absence of treatments we can use this architecture for estimation of hazards and survival functions by using only one treatment $a=0$. As we show in the experiments, this version of our method -- SurvIHE (Individualized Hazard Estimator for Survival analysis) -- is of independent interest in the standard survival setting, as it tackles Shifts 2 \& 3. Below, we describe the empirical loss functions we use to find representation $\Phi$ and hypotheses $h_{a, \tau}$.

Let $\Phi:\Xc \rightarrow \mathcal{R}$ denote the \textit{representation} (parameterized by $\theta_{\phi}$) and $h_{a, \tau}: \mathcal{R} \rightarrow [0,1]$ the \textit{hazard estimator} for treatment $a$ and time $\tau$ (parameterized by $\theta_{h_{a, \tau}}$), each implemented as a fully-connected neural network. While the output heads are thus unique to each treatment-group time-step combination, we allow hazard estimators to share information by using \textit{one} shared representation for all hazard functions. This allows for both borrowing of information across different $a, \tau$ and significantly reduces the number of parameters of the network. Then, given the time-to-event data $\data$, we use the following empirical loss functions for the observational risk and the IPM term: 
\begin{equation} \nonumber
\begin{split}
    \loss_{risk}(\theta_{\phi}, \theta_{h}) &= \frac{1}{\tmax}\sum_{t=1}^{\tmax}\!\sum_{i: \tilde{\tau}_{i} \geq t} n^{-1}_{1, t} a_{i}\ell\big(y_{i}(t), {h}_{1,t}({\Phi}(x_{i}) ) \big) + n^{-1}_{0, t}(1\minus a_{i})\ell\big(y_{i}(t), {h}_{0,t}({\Phi}(x_{i}) ) \big), \\
    \loss_{ipm}(\theta_{\phi}) &= \sum_{a\in \{0,1\}} \sum_{t=1}^{\tmax} Wass\big( \{ {\Phi}(x_{i}) \}_{i=1}^{n}, \{  {\Phi}(x_{i}) \}_{i:\tilde{\tau}_{i} \geq t, a_{i}=a} \big),
\end{split}
\end{equation}
where $Wass(\cdot, \cdot)$ is the finite-sample Wasserstein distance \cite{cuturi2014wass}; refer to Appendix D for further detail. Note that $\loss_{ipm}(\theta_{\phi})$, which penalizes the discrepancy between the baseline distribution and \textit{each} at-risk distribution $\mathbb{P}^{\Phi}_{a, \tau}$, simultaneously tackles all three sources of shifts. Further, $n_{a, t}=|\mathcal{I}(\tau, a)|$ is the number of samples at-risk in each treatment arm; its presence ensures that each $a, \tau$-combination contributes equally to the loss. Overall, we can find ${\Phi}$ and ${h}_{a,\tau}$'s that optimally trade off balance and predictive power as suggested by the generalization bound \eqref{eq:mainbound} by minimizing the following loss:
\begin{equation} \label{eq:loss_target}
    \loss_{target}(\theta_{\phi}, \theta_{h}) = \loss_{risk}(\theta_{\phi}, \theta_{h}) + \beta \loss_{ipm}(\theta_{\phi})
\end{equation}
where $\theta_{h} = \{\theta_{h_{a, \tau}}\}_{a\in\{0,1\}, \tau \in \Tc} $, and $\beta > 0$ is a hyper-parameter.  
The pseudo-code of \proposed, the details of how to obtain $Wass(\cdot, \cdot)$ and how we set $\beta$ can be found in Appendix D.

\section{Related work}
\textbf{Heterogeneous treatment effect estimation (non-survival)} has been studied in great detail in the recent ML literature. While early work built mainly on tree-based methods \cite{hill2011bayesian, athey2016recursive, wager2018estimation, athey2019generalized}, many other methods, such as Gaussian processes \cite{alaa2017bayesian, alaa2018limits} and GANS \cite{yoon2018ganite}, have been adapted to estimate HTEs. Arguably the largest stream of work \cite{johansson2016learning, Shalit:16, johansson2018learning, shi2019adapting, hassanpour2019counterfactual, hassanpour2020learning, assaad2021counterfactual, curth2020} built on NNs, due to their flexibility and ease of manipulating loss functions, which allows for easy incorporation of balanced representation learning as proposed in \cite{johansson2016learning, Shalit:16} and motivated also the approach taken in this paper. Another popular approach has been to consider model-agnostic (or `meta-learner' \cite{kunzel2019metalearners}) strategies, which provide a `recipe' for estimating HTEs using \textit{any} predictive ML method \cite{kunzel2019metalearners, kennedy2020optimal, nie2021quasi, curth2020}. Because of their simplicity, the \textit{single model} (S-learner) -- which uses the treatment indicator as an additional covariate in otherwise standard model-fitting -- and \textit{two model} (T-learner) -- which splits the sample by treatment status and fit two separate models -- strategies \cite{kunzel2019metalearners}, can be directly applied to the survival setting by relying on a standard survival (prediction) method as base-learner. 

\textbf{ML methods for survival prediction} continue to multiply; here we focus on the most related class of methods -- namely on those nonparametrically modeling conditional hazard or survival functions  --  and \textit{not} on those relying on flexible implementations of the Cox proportional hazards model (e.g. \cite{faraggi1995neural, Katzman:16, Luck:17}) or modeling (log-)time as a regression problem (e.g. \cite{Hothorn:06, ranganath:16, chapfuwa2018adversarial, steingrimsson2019censoring, steingrimsson2020deep, avati2020countdown}). One popular nonparametric estimator of survival functions is \cite{Ishwaran:08}'s random survival forest, which relies on the Nelson-Aalen estimator to nonparametrically estimate the cumulative hazard within tree-leaves.  The idea of modeling discrete-time hazards directly using \textit{any arbitrary classifier} and long data-structures goes back to at least \cite{brown1975use}, with implementations using NN-based methods presented in e.g. \cite{biganzoli1998feed, gensheimer2019scalable, ren2017dsra, kvamme2019continuous}.  \cite{Changhee:AAAI18} models the probability mass function instead of the hazard, and \cite{yu2011learning} use labels $\mathbbm{1}\{T>t\}_{t \in \mathcal{T}}$ to estimate the survival function directly using multi-task logistic regression. For a more detailed overview of different strategies for estimating survival functions, refer to Appendix B. 

\textbf{Estimating HTEs from time-to-event data} has been studied in much less detail. \cite{ tabib2020non, zhang2017mining} use tree-based nearest-neighbor estimates to estimate expected differences in survival time directly, and \cite{henderson2020individualized} use a BART-based S-learner to output expected differences in log-survival time. \cite{hu2020estimating} performed a simulation study using different survival prediction models as base-learners for a two-model approach to estimating the difference in median survival time. Based on ideas from the semi-parametric efficiency literature, \cite{cui2020estimating}
and \cite{diaz2018targeted} propose estimators that target the (restricted) mean survival time \textit{directly} and consequently \textit{do not} output estimates of the treatment-specific hazard or survival functions. We consider the ability to output treatment-specific predictions an important feature of a model if the goal is to use model output to give decision support, given that it allows the decision-maker to trade-off relative improvement with the baseline risk of a patient. Finally, \cite{chapfuwa2021enabling} recently proposed a generative model for treatment-specific event times which relies on balancing representations to balance only the treatment groups at baseline. This model does not output hazard- or survival functions, but can provide approximations by performing Monte-Carlo sampling.

\section{Experiments}
Unfortunately, when the goal is \textit{estimating} (differences of) survival functions (instead of \textit{predicting} survival), evaluation on real data will not reflect performance w.r.t. the intended baseline population. 
Therefore, we conduct a range of synthetic experiments with \textit{known} ground truth. We evaluate the effects of different shifts separately by starting with survival estimation \textit{without} treatments, and then introduce treatments. Finally, we use the real-world dataset Twins \cite{louizos2017causal} which has uncensored survival outcomes for twins (where the treatment is `being born heavier'), and is hence free of Shifts 1 \& 2. 

\paragraph{Baselines.} 
Throughout, we use Cox regression (\textbf{Cox}), a model using a separate logistic regression to solve the hazard classification problem at each time-step (\textbf{LR-sep}), random survival forest (\textbf{RSF}), { and a deep learning-based time-to-event method \cite{Changhee:AAAI18} (\textbf{DeepHit})} as natural baselines; when there are treatments, we use them in a two-model (T-learner) approach. In settings with treatments, we additionally use the CSA-INFO model of \cite{chapfuwa2021enabling} (\textbf{CSA}), where we use its generative capabilities to approximate target quantities via monte-carlo sampling. Finally, we consider ablations of SurvITE (and SurvIHE); in addition to removing the IPM term (\textbf{SurvITE (no IPM)}), we consider two variants of SurvITE based on \cite{Shalit:16}'s CFRNet balancing term: \textbf{SurvITE (CFR-1)} creates a representation balancing treatment groups at baseline only, and \textbf{SurvITE (CFR-2)} creates a representation optimizing for balance of treatment groups \textit{at each time step} (i.e. no balancing towards $\mathbb{P}_0$). We discuss implementation in Appendix D.

\paragraph{Synthetic Experiments.}
We consider a range of synthetic simulation setups (S1-S4) to  highlight and isolate the effects of the different types of covariate shift. As event and censoring processes, we use
\begin{equation*}
    \lambda^a(t|x) =
        \begin{cases}
        0.1\sigma(-5x_{1}^{2} - a\cdot(\mathbbm{1}\{x_3\geq 0\}+0.5))&\!\text{for}~t \leq 10\\
        0.1\sigma(10x_{2}  - a\cdot(\mathbbm{1}\{x_3\geq 0\}+0.5)))&\!\text{for}~t > 10 
        \end{cases},~~~~~\lambda_{C}(t|x) =  0.01\sigma(10x_{4}^{2})
\end{equation*}
with treatment assignment mechanism  $a \sim \texttt{Bern}(\xi\cdot\sigma(\sum\nolimits_{p \in \mathcal{P}}x_{p}))$, with $\sigma$ the sigmoid function.   Additionally, we assume administrative censoring at $t=30$ throughout, i.e., $\lambda_{C}(30|x) = 1$, marking e.g. the end of a hypothetical clinical study. Covariates are generated from a 10-dimensional multivariate normal distribution with correlations, i.e. $X \sim \Norm(\textbf{0},\bm{\Sigma})$ where $\bm{\Sigma} = (1-\rho) \Iv + \rho \textbf{1}\textbf{1}^{\top}$ with $\rho = 0.2$. We use 5000 independently generated samples each for training and testing.

 \begin{figure*}[t]
    \centering
    \includegraphics[width=\textwidth]{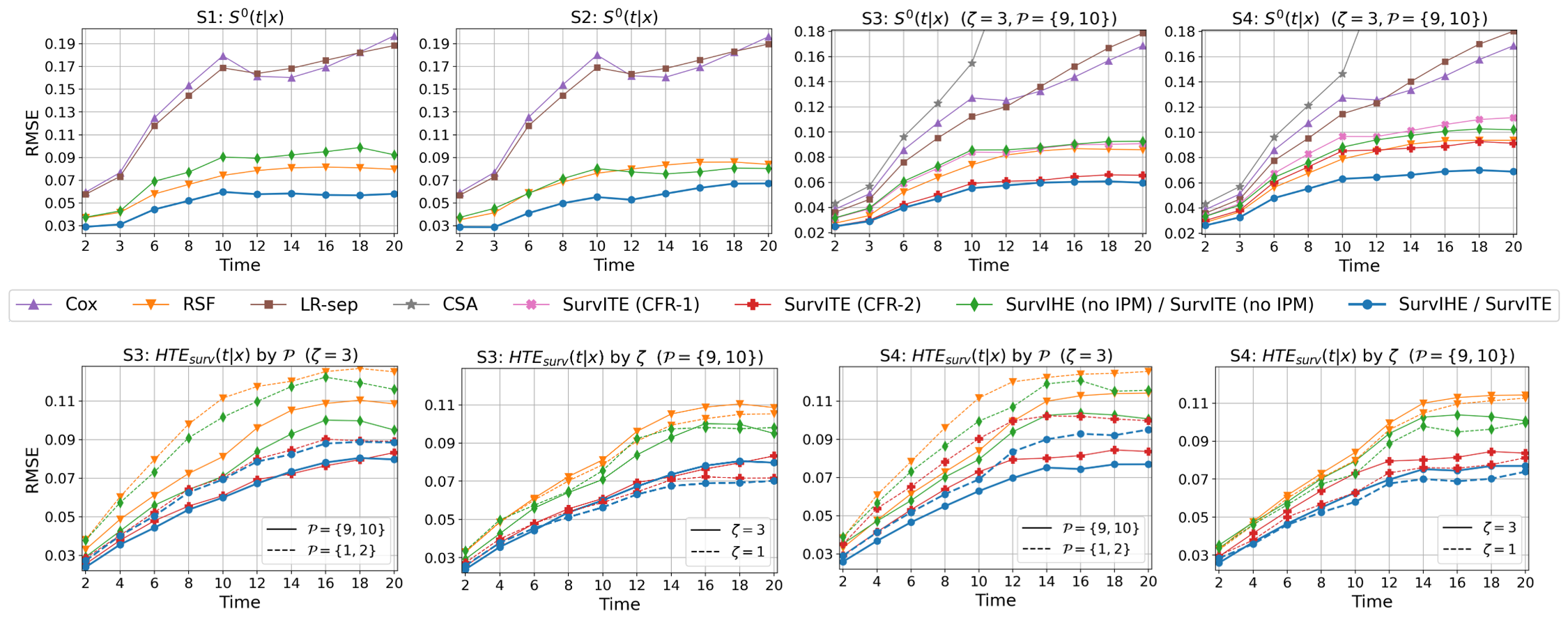}
        \vspace{-0.1in}
	\caption{RMSE of estimating the survival function $S^0(t|x)$ (top) and the treatment effect $HTE_{surv}(t|x)$ (bottom)  for different time steps across synthetic settings. Averaged across 5 runs. } \label{fig:survivalfunc} 
          \vspace{-0.1in}  
\end{figure*}

In S1, we begin with the simplest case -- \textit{no} treatments and \textit{no} censoring -- using only $\lambda^0(t|x)$ to generate events, considering only event-induced shift (Shift 3). In S2, we introduce informative censoring using $\lambda_{C}(t|x)$ (Shift 2+3). In S3, we use treatments and consider biased treatment assignment  (without censoring) (Shift 1+3). In S4, we consider the most difficult case with all three types of shift (Shift 1+2+3). In the latter two settings, we vary treatment selection by changing (i) whether the covariate set overlaps with the event-inducing covariates ($\mathcal{P}{=}\{1, 2\}$) or not ($\mathcal{P}{=}\{9, 10\}$) and (ii) the selection strength $\xi \in \{1, 3\}$. We present exploratory plots of these DGPs in Appendix E.

Fig. \ref{fig:survivalfunc} (top) shows performance on estimating $S^0(t|x)= \prod\nolimits_{k \leq t}\big(1 - \lambda^{0}(k| x)\big)$ for all scenarios and methods, while Fig. \ref{fig:survivalfunc} (bottom) shows performance on estimating the difference in survival functions ($HTE_{surv}(t|x)$) for a selection of methods (for readability, full results in Appendix F). In Table \ref{Table:RMST}, we further evaluate the estimation of differences in RMST ($\text{HTE}_{rmst}(x)$).  
Results for hazard estimation and additional performance metrics for survival \textit{prediction} are reported in Appendix F. We observe that SurvITE (/SurvIHE) performs best throughout, and that introduction of the IPM term leads to substantial improvements across all scenarios. In S1 with only event-induced covariate shift and in S3/4 when treatment selection and event-inducing covariates overlap ($\mathcal{P}{=}\{1, 2\}$), balancing cannot remove all shift as the shift-inducing covariates are predictive of outcome; however, even here the IPM-term helps as it encourages dropping other covariates (which appear imbalanced due to correlations in $X$). While our method was motivated by theory for estimation of hazard functions, it thus indeed also leads to gains in survival function estimation. 
As expected, both Cox and LR-sep do not perform well as they are misspecified, while the nonparametric RSF is sufficiently flexible to capture the underlying DGP and usually performs similarly to SurvITE (architecture only), but is outperformed once the IPM term is added. For readability, we did not include DeepHit in Fig. \ref{fig:survivalfunc}; using table F.1 presented in Appendix F, we  observe that DeepHit performs worse than the SurvITE architecture without IPM term, indicating that our model architecture alone is better suited for estimation of treatment-specific survival functions (note that \cite{Changhee:AAAI18} focused mainly on discriminative (predictive) performance, and not on the estimation of the survival function itself). Therefore, upon addition of the IPM-terms, the performance gap between SurvITE and DeepHit only becomes larger.

A comparison with ablated versions highlights the effect of using the baseline population to define balance; naive balancing across treatment arms (either at baseline -- SurvITE(CFR-1), or over time -- SurvITE(CFR-2)) is not as effective as using the baseline population as a target, especially at the later time steps where the effects of time-varying shifts worsen. While SurvITE(CFR-2) almost matches the performance of the full SurvITE in S3, it performs considerably worse in S4, indicating that this form of balancing suffers mainly due to its ignorance of censoring.  Finally, a comparison with CSA highlights the value of modeling hazard functions directly: we found that Monte-Carlo approximation of the survival function using the generated event times gives very badly calibrated survival curves as event times generated by CSA were concentrated in a very narrow interval, leading to survival estimates of 0 and 1 elsewhere. Its performance on estimation of RMST was likewise poor; we conjecture that this is due to (i) CSA modeling continuous time, while the outcomes were generated using a coarse discrete time model, and (ii) the significant presence of administrative censoring.

\begin{table}[t!]
	\caption{RMSE on estimation of $\text{HTE}_{rmst}(x)$  (mean $\pm$ 95\%-CI) for different times for the Synthetic and Twins datasets ($L$s are the 25 \& 75th and 75 \& 95th percentiles of event times, respectively).} \label{Table:RMST}
	\begin{center}
	\fontsize{6.3}{7}\selectfont
        \begin{tabular}{c c c c c c c c c}
        \toprule
        \multirow{2}{*}{\textbf{Methods}}
        &\multicolumn{2}{c}{\textbf{S3} ($\zeta=3$, no overlap)}
        &\multicolumn{2}{c}{\textbf{S4} ($\zeta=3$, no overlap)}
        &\multicolumn{2}{c}{\textbf{Twins (no censoring)}} 
        &\multicolumn{2}{c}{\textbf{Twins (censoring)}} \\ 
        &$L=10$&$L=20$&$L=10$&$L=20$&$L=30$&$L=180$&$L=30$&$L=180$\\\midrule
        Cox         &0.434$\pm$0.03 &1.073$\pm$0.05 &0.424$\pm$0.02
                &1.047$\pm$0.04 &2.85$\pm$0.10 &20.33$\pm$0.50&2.88$\pm$0.09 &20.60$\pm$0.50\\ 
        RSF         &0.328$\pm$0.02 &1.027$\pm$0.03 &0.332$\pm$0.02 &1.058$\pm$0.03 &3.15$\pm$0.07 &22.42$\pm$0.36 &3.18$\pm$0.08 &22.62$\pm$0.46\\
        LR-sep      &0.412$\pm$0.02 &1.111$\pm$0.07 &0.418$\pm$0.02 &1.149$\pm$0.04 &2.94$\pm$0.10 &20.60$\pm$0.53 &2.94$\pm$0.10 &20.66$\pm$0.52\\
        DeepHit	    &0.347$\pm$0.03 &0.821$\pm$0.07 &0.361$\pm$0.08 &0.830$\pm$0.15 &2.95$\pm$0.28 &20.89$\pm$1.91 &2.86$\pm$0.09 &20.69$\pm$0.52\\
        CSA         &0.421$\pm$0.01 &2.098$\pm$0.26 &0.406$\pm$0.01 &1.932$\pm$0.12 &3.42$\pm$0.12 &26.20$\pm$1.21 &4.41$\pm$0.54 &47.79$\pm$1.55\\  \midrule
        \proposed~(no IPM) &0.275$\pm$0.04 &0.843$\pm$0.11 &0.310$\pm$0.05 &0.930$\pm$0.11 &2.80$\pm$0.10 &19.80$\pm$1.01 &2.85$\pm$0.22 &20.00$\pm$1.07 \\ 
        \proposed~(CFR-1)   &0.269$\pm$0.04 &0.825$\pm$0.09 &0.341$\pm$0.02 &1.016$\pm$0.10 &2.68$\pm$0.06 &19.16$\pm$0.37 &2.67$\pm$0.15 &19.10$\pm$0.85\\ 
         \proposed~(CFR-2)    &0.236$\pm$0.04 &0.691$\pm$0.08 &0.294$\pm$0.07 &0.815$\pm$0.15 &2.61$\pm$0.12 &18.69$\pm$0.64 &2.69$\pm$0.22 &19.20$\pm$1.44\\
        \textbf{\proposed} &\textbf{0.225$\pm$0.03} &\textbf{0.687$\pm$0.08} &\textbf{0.237$\pm$0.03} &\textbf{0.703$\pm$0.06} &\textbf{2.53$\pm$0.09} &\textbf{18.34$\pm$0.70} &\textbf{2.63$\pm$0.10} &\textbf{18.76$\pm$0.56}\\ \bottomrule
        \end{tabular}
	\end{center}
	        \vskip -0.1in
\end{table}

\paragraph{Real data: Twins.}\label{sec:twins}
Finally, we consider the Twins benchmark dataset, containing survival times (in days, administratively censored at t=365) of 11400 pairs of twins, which is used in \cite{louizos2017causal, yoon2018ganite} to measure HTEs of birthweight on infant mortality. We split the data 50/50 for training and testing (by twin pairs), and similar to \cite{yoon2018ganite}, use a covariate-based sampling mechanism to select only one twin for training to emulate selection bias. Further, we consider a second setting where we additionally introduce covariate-dependent censoring. For all discrete-time models, we use a non-uniform discretization to construct classification tasks because most events are concentrated in the first weeks. A more detailed description of the data and experimental setup can be found in Appendix E. As the data is real and ground truth probabilities are unknown, $\text{HTE}_{rmst}(x)$ is suited best to evaluate performance on estimating effect heterogeneity. The results presented in Table \ref{Table:RMST} largely confirm our findings on relative performance in the synthetic experiments; only RSF performs relatively worse on this dataset. 

\section{Conclusion} We studied the problem of inferring heterogeneous treatment effects from time-to-event data by focusing on the challenges inherent to treatment-specific hazard estimation. We found that a variety of covariate shifts play a role in this context, theoretically analysed their impact, and demonstrated across a range of experiments that our proposed method \proposed~successfully mitigates them.

\textbf{Limitations.} Like all methods for inferring causal effects from observational data, \proposed~relies on a set of strong assumptions which should be evaluated by a domain expert prior to deployment in practice. Here, the time-to-event nature of our problem adds an additional assumption (`random censoring') to the standard `no hidden confounders' assumption in classical treatment effect estimation. If such assumptions are not properly assessed in practice, any causal conclusions may be misleading. 

\clearpage

\acksection
We thank anonymous reviewers as well as members of the vanderschaar-lab for many insightful comments and suggestions. AC gratefully acknowledges funding from AstraZeneca. CL was supported through the IITP grant funded by the Korea government(MSIT) (No. 2021-0-01341, AI Graduate School Program, CAU). Additionally, MvdS received funding from the Office of Naval Research (ONR) and the National Science Foundation (NSF, grant number 1722516).

\bibliography{arxiv_version}
\bibliographystyle{unsrt}
\newpage

\newcommand{\newnumbering}{
\setcounter{figure}{0}\renewcommand{\thefigure}{\thesection.\arabic{figure}}
\setcounter{table}{0}\renewcommand{\thetable}{\thesection.\arabic{table}}
\setcounter{equation}{0}\renewcommand{\theequation}{\thesection.\arabic{equation}}
}

\appendix
\section*{Appendix}
This appendix is organized as follows: We first present an extended overview of the standard treatment effect estimation setup and discuss differences with the time-to-event setting (Appendix A). Then, we give an extended review of strategies for nonparametric estimation of survival dynamics (Appendix B). In Appendix C we discuss technical details -- assumptions and proofs -- and Appendix D we discuss implementation. Appendix E contains additional descriptions of datasets and experimental setup and Appendix F presents  additional results.
\section{Preliminaries on treatment effect estimation} 
In the standard treatment effect estimation setup with binary or continuous outcomes (see e.g. \cite{Shalit:16, alaa2018limits, curth2020}), one usually observes a dataset $\mathcal{D}=\{(a_i, x_i, y_i)\}^n_{i=1}$ comprising $n$ realizations of the tuple $(A, X, Y)$. $X\in \mathcal{X}$ and $A \in \{0, 1\}$ represent patient characteristics and treatment assignment as in the main text. $Y\in\mathcal{Y}$ is usually a binary ($\mathcal{Y}=\{0, 1\}$) or continuous ($\mathcal{Y}=\mathbb{R}$) outcome. The target parameter of interest is often the \textit{conditional average treatment effect (CATE)}
\begin{equation}
    \tau(x) = \mathbb{E}[Y|X=x, do(A=1)] - \mathbb{E}[Y|X=x, do(A=0)]
\end{equation}
which is impossible to estimate from observational data without further assumptions, as -- due to the \textit{fundamental problem of causal inference} \cite{holland1986statistics} -- every individual is only ever observed under \textit{one} of the two possible interventions. CATE can therefore only be nonparametrically estimated under the imposition of untestable assumptions;
here we rely on the standard ignorability assumptions \cite{rosenbaum1983central} of \textit{No hidden confounders (1.a)}, \textit{Consistency (1.c)} and \textit{Positivity/Overlap in treatment assignment (2.a)}.

\subsection{Comparison with the time-to-event treatment effects setup}
The time-to-event setting is made more involved by (i) the presence of censoring and (ii) the interest in the \textit{dynamics} of the underlying survival process.

Censoring -- the removal of some individuals from the sample before having observed their event time -- further complicates the treatment effect estimation problem, because every individual's outcome (time-to-event) is now observed under \textit{at most} one intervention. The presence of censoring adds an additional source of covariate shift, and the need to rely on the assumptions of \textit{Censoring at random (1.b)} and \textit{Positivity in censoring (2.b)}. Censoring is, however,  different from complete missingness of the outcome as the censoring time provides \textit{some} information on the outcome -- an individual has survived \textit{at least} until the censoring time. 

While the difference in expected survival time (the time-to-event equivalent in CATE) can be the treatment effect of interest in a study, many survival analysis problems are concerned with target parameters that capture differences in the \textit{dynamics} of the underlying survival process across treatments, e.g. hazard ratios or differences in survival functions -- which substantially increases the number of possible target parameters to model (beyond `only' CATE). Instead of only modeling expected outcomes (as would be the case in the standard setup as discussed above), modeling survival dynamics through e.g. the treatment-specific hazard function can therefore  often be of interest. Nonparametrically modeling hazard functions introduces the additional assumption on \textit{Positivity of events (2.c)}.
\section{Strategies for loss-based discrete-time hazard and survival function estimation} \newnumbering

In this section, we review strategies for nonparametric (or machine-learning based) estimation of the dynamics underlying discrete-time event processes. Here, we consider on the standard case \textit{without treatments} to highlight how a dependence on different populations arises in different modeling strategies, and follow closely the exposition of different strategies in \cite{kvamme2019continuous}. We focus on loss functions that can be used for implementation to highlight that these approaches are valid for use of \textit{any} classifier, and then briefly mention specific instantiations of such approaches from related work. 

\textbf{Preliminaries.}  In addition to hazard and survival function defined in the main text, define the probability mass  functions (PMF) as 
\begin{equation}
    f(\tau|x) = \PR(T=\tau|X=x) \text{ and }  f_C(\tau|x) = \PR(C=\tau|X=x) 
\end{equation}
Note that a hazard $\lambda(\tau|x)=\PR(T=\tau|T\leq \tau, X=x)$ can then also be defined as $\lambda(\tau|x)=\frac{f(\tau|x)}{S(\tau-1|x)}$. Further, recall that the survival function $S(\tau|x)=\prod_{t \leq \tau}\big(1 - \lambda(t|x)\big)$, so that the PMF can be rewritten as $f(\tau|x)=\lambda(\tau|x)S(\tau-1|x)=\lambda(\tau|x)\prod_{t \leq \tau-1}\big(1 - \lambda(t|x)\big)$.

\subsection{Likelihood-based hazard estimation}
Under the assumption of random censoring (which is discussed further in Appendix \ref{sec:assumptions}), the likelihood function of the observed (short) data factorizes; i.e. 
\begin{equation*}
\begin{split}
    &\PR(\tilde{T}=\tilde{\tau}, \Delta=\delta|X=x) = \PR(T=\tilde{\tau}, C \geq \tilde{\tau}|X=x)^\delta  \PR(T > \tilde{\tau}, C =\tilde{\tau}|X=x)^{1-\delta}  \\
    &~~~~~~~~~~~~~~~~~~=\big[\PR(T=\tilde{\tau}|X=x) \PR(C\geq \tilde{\tau}|X=x)\big]^\delta \big[\PR(T>\tilde{\tau}|X=x) \PR(C= \tilde{\tau}|X=x)\big]^{1-\delta}  \\
    &~~~~~~~~~~~~~~~~~~= \big[f(\tilde{\tau}|x)(S_C(\tilde{\tau}|x) + f_C(\tilde{\tau}|x))\big]^\delta \big[S(\tilde{\tau}|x)f_C(\tilde{\tau}|x))\big]^{1-\delta} \\
    &~~~~~~~~~~~~~~~~~~= \underbrace{f(\tilde{\tau}|x)^\delta S(\tilde{\tau}|x)^{1-\delta}}_{\text{Event-relevant}} \underbrace{f_C(\tilde{\tau}|x)^{1-\delta}(S_C(\tilde{\tau}|x)+f_C(\tilde{\tau}|x))^\delta}_{\text{Ignorable censoring mechanism}}
\end{split}
\end{equation*}
By the likelihood principle, the parts pertaining to censoring are \textit{ignorable}, hence we can consider censoring and event likelihoods separately \cite{rubin1976inference}. The likelihood contribution of observation $i$ to the negative time-to-event likelihood can then be written as:
\begin{equation}
    L_i= - f(\tilde{\tau}_i|x_i)^{\delta_i} S(\tilde{\tau}_i|x_i)^{1-\delta_i}
    = \lambda(\tilde{\tau}_i|x)^{\delta_i} (1-\lambda(\tilde{\tau}_i|x))^{1-\delta_i} \prod_{t \leq \tilde{\tau}_i-1}\big(1 - \lambda(t|x)\big)
\end{equation}
so that, after taking the logarithm and summing over all $i\in [n]$ we have that
\begin{equation} \label{eq:loss_longformat}
\begin{split}
\mathcal{L} =  - \sum^n_{i=1} \big(\delta_i log(\lambda(\tilde{\tau}_i|x))+(1-\delta_i)log(1-\lambda(\tilde{\tau}_i|x)) + \sum_{t \leq \tilde{\tau}_i-1} log(1-\lambda(t|x))\big)\\
=-  \sum^\tmax_{t=1} \sum^n_{i=1} \one(\tilde{\tau}_i \geq t) \big(y_i(t) log(\lambda(t|x))+(1-y_i(t))log(1-\lambda(t|x))\big)
\end{split}
\end{equation}
with $y_i(t)=\one(\tilde{\tau}_i=t, \delta_i=1) = \one(N_{T}(t)_i=1 \cap N_{T}(t\minus 1)_i=0)$ as in the main text. Thus, the classification approach with log-loss is \textit{equivalent} to optimizing for the likelihood of the hazard. Optimizing the likelihood of the hazard thus suffers from the exact same shifts as the classification approach, namely the shifts induced by focusing on the `at-risk' population at any time-step: the log-loss also has dependence on $\one(\tau_i \geq t)$. Note that, as we illustrate in section \ref{sec:misspecIllustration}, under such shifts, optimizing the likelihood is only problematic if the model for $\lambda(\tau|x)$ is misspecified -- a well-established fact in the literature on covariate shift \cite{shimodaira2000improving}.

Depending how $\lambda(\tau|x)$ is parameterized, different models proposed in related work arise. The idea to use a classification approach dates back to at least the logistic-hazard model in \cite{brown1975use}, and is reviewed in more detail in \cite{tutz2016modeling}. The first NN-based implementation that we are aware of is \cite{biganzoli1998feed}, which parameterizes $\lambda(\tau|x)$ by using one shared network for all $\tau \in \Tc$ where the time-indicator $\tau$ is passed as an additional covariate. \cite{gensheimer2019scalable} instead propose a network with some shared layers and $\tau$-specific output layers (resulting in a model similar to the SurvIHE base-model). Finally, \cite{ren2017dsra}'s DSRA parameterizes $\lambda(\tau|x)$ using a recurrent network which encodes the structure shared across time.

\subsubsection{Illustration: Why (mis)specification matters}\label{sec:misspecIllustration}
To briefly illustrate when event-induced at-risk population shift matters, we consider two simple toy examples: we rely on event-processes with covariate-dependent but time-constant hazards, i.e. $\lambda(\tau_1|x)=\lambda(\tau_2|x)$, and there are 5 multivariate normal correlated covariates, of which only $X_1$ determines the hazard. We parameterize hazard estimators using a separate logistic regression at each time step $t$. We consider one process where this logistic regression is correctly specified for the underlying hazard function, as $\lambda_1(\tau|x)=\sigma(x_1 - 0.25)$. We consider another process where this logistic regression is misspecified, as $\lambda_2(\tau|x)=\sigma(\one(x_1>0)x_1 - 0.25)$ (i.e. there is a nonlinearity that cannot be perfectly captured by a simple logistic regression).

\begin{figure*}[!h]
	\centering
	\begin{subfigure}[b]{0.4\linewidth}
		\centering
		\includegraphics[width=1.0\linewidth]{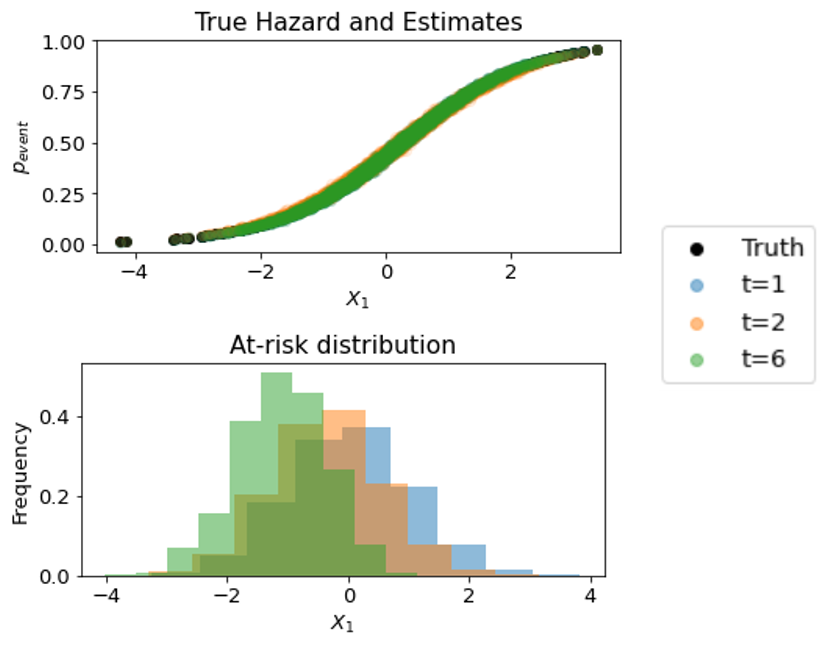} 
		\caption{Hazard estimates and at-risk distribution for $\lambda_1(\tau|x)$ (well-specified model)}
	\end{subfigure}~
	\begin{subfigure}[b]{0.4\linewidth}
		\centering
		\includegraphics[width=1.0\linewidth]{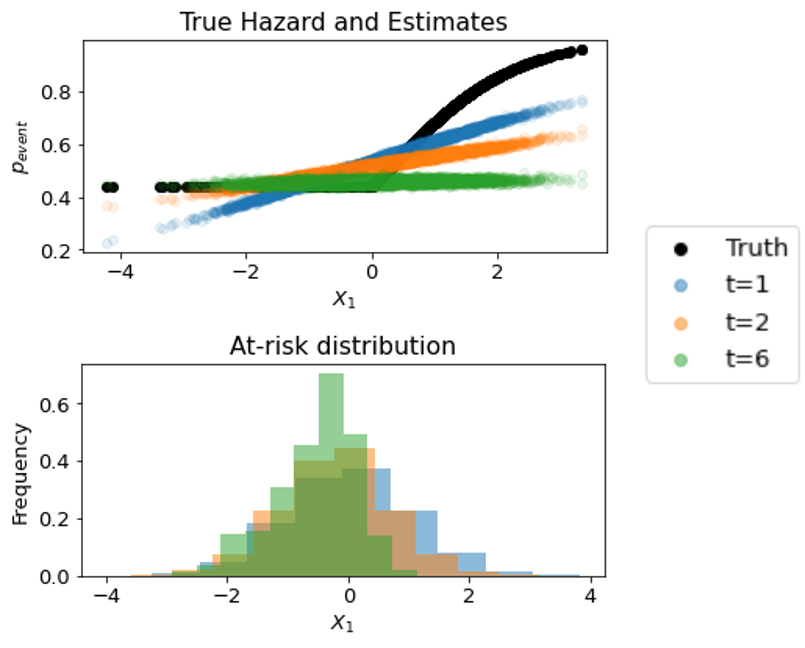} 
		\caption{Hazard estimates and at-risk distribution for $\lambda_2(\tau|x)$ (misspecified model)}
	\end{subfigure}~
	\caption{Toy example highlighting that covariate shift plays no role when models are well-specified (left) but matters under misspecification (right).} \label{fig:toyex} \vspace{-3mm}
\end{figure*}

As can be seen in Fig. \ref{fig:toyex}, both processes lead to event-induced covariate shift. However, this shift has no effect on hazard estimator performance over time when the model is correctly specified. Yet, when the model is incorrectly specified, the estimator has to trade off making errors in different regions of the covariate space. The optimal trade-off w.r.t. the baseline distribution is made by the hazard classifier at $t=1$ where the at-risk distribution corresponds to the marginal distribution of covariates. Due to event-induced covariate shift, hazard estimates become increasingly biased towards the survivor population at later time-steps. 

\subsection{Survival-based estimation} 
An alternative approach to targeting the likelihood of the hazard would be to target the survival function directly by realizing that $\PR(T>t|X=x)=\mathbb{E}[\one(T>t)|X=x]$, so that the survival function can be estimated directly by solving $\tmax$ classification problems with targets $\{\one(T>t)\}_{t \in \Tc}$. This considers a loss function
\begin{equation}
    \mathcal{L} =  - \sum^n_{i=1} \sum_{t \in \Tc}  \one(\tilde{\tau}_i > t)  \log(S(t|x_i)) + \delta_i(1 - \one(\tilde{\tau}_i > t)) \log(1-S(t|x_i))
\end{equation}
which suffers from censoring-induced covariate shift due to the interaction of $\delta_i(1 - \one(\tilde{\tau}_i > t))$; i.e. only non-censored individuals contribute to the `negative class' $\one(\tilde{\tau}_i \leq t)$, an effect that gets larger for $t$ large. 

The multi-task logistic regression approach proposed in \cite{yu2011learning} is a variant of the more general approach described above; it uses a modeling approach based on conditional random fields \cite{lafferty2001conditional} and jointly models all survival functions by accounting for the sequential nature of targets $\{\one(T>t)\}_{t \in \Tc}$ and the existance of a restricted set of `legal' values.

\subsection{PMF-based estimation}
Finally, instead of focussing on hazard or survival function, one could also estimate the PMF function; the PMF can be transformed to hazard or survival functions by realizing that $\lambda(1|x) = f(1|x)$ and $S(1|x)=(1-f(1|x))$. This can be done by treating the survival problem as a $\tmax$-class classification problem with one-hot encoded labels $(\one(\tau_i = t))_{t \in \Tc}$ leading to the loss
\begin{equation}
    \mathcal{L} =  - \sum^n_{i=1} \delta_i\log(f(\tilde{\tau}_i|x_i)) 
\end{equation}
so that each uncensored observation contributes mainly to the estimate of $f$ at its event-time step $\tau_i$  \cite{ren2017dsra} (instead of multiple time-steps as in the previous two subsections). Due to the presence of censoring indicator $\delta_i$, this suffers from censoring-induced covariate shift. As in \cite{Changhee:AAAI18}'s DeepHit, a likelihood contribution $(1-\delta_i) \log\big(\sum^\tmax_{t=\tilde{\tau}_i+1} f(t|x_i)\big)$ marginalizing over possible outcomes for all censored observations can be added, such that they contribute to $t> \tau_i$ by signalling that their event times are larger. For correctly specified models $f$ this corresponds to optimizing the likelihood of the PMF and is hence sufficient to correct for censoring, however, otherwise this does not exactly correct for censoring-induced covariate shift. 

\section{Technical details: Assumptions and Proofs} \newnumbering

\subsection{Assumptions}\label{sec:assumptions}
In this section, we discuss and formally state the assumptions made in Section 2. As e.g. \cite{stitelman2010collaborative, stitelman2011targeted, cai2019targeted}, we assume the fairly general causal structure encoded in the DAG in Figure 1. By assuming that observed data was generated from this DAG, the classical identifying assumptions (No Hidden Confounders, Censoring At Random, and Consistency) are implicitly formalized \cite{stitelman2010collaborative}. 

Equivalently, we can restate the assumptions using potential outcomes \cite{rubin2005causal} notation.  As in e.g. \cite{diaz2018targeted}, we let $T_a$ denote the potential event time that would have been observed had treatment a been assigned, and $C=\tmax$ been externally set. Then, the following assumptions are implied by the DAG:

\begin{assumption}[1.a No Hidden Confounders/ Unconfoundedness] Treatment assignment is random conditional on covariates, i.e. $T_a  \indep A | X$.
\end{assumption}

\begin{assumption}[1.b Censoring at random] Censoring and outcome are conditionally independent, i.e.  $T_a \indep C | X, A$.
\end{assumption}

\begin{assumption}[1.c Consistency] The observed outcomes are the potential outcomes under the observed intervention, i.e. if $A=a$ then $T=T_a$.
\end{assumption}

Then, we can write 
\begin{equation*}
\begin{split}
 \lambda^{a}(\tau|x) &= \PR(T=\tau | T \geq \tau, do(A=a, C\geq \tau), X=x)\\
 &= \PR(T_a=\tau | T_a \geq \tau, do(C\geq \tau), X=x) \\
 &= \PR(T_a=\tau | T_a \geq \tau, A=a, do(C\geq \tau), X=x) \\
 &=  \PR(T_a=\tau | T_a \geq \tau, C\geq \tau, A=a, X=x) \\
 &= \PR(T=\tau | T \geq \tau,    C\geq \tau, A=a, X=x) \\
 &=  \PR(\tilde{T}=\tau, \Delta=1|\tilde{T} \geq \tau, A=a, X=x) = \lambda(\tau|a,x) 
\end{split}
\end{equation*}
Here, the equalities in line one and two follow by definition, line three follows by assumption 1.a, line four follows by assumption 1.b,  the equality in line five follows by assumption 1.c, and the final line follows by definition. 

To enable nonparametric estimation of $\lambda^{a}(\tau|x)$ for some fixed $\tau \in \Tc$, we additionally consider a number of conditions on the likelihood of observing certain events.

\begin{assumption}[2.a Overlap/positivity (treatment assignment)] Treatment assignment is non-deterministic, i.e. for some $\epsilon_1>0$, we have that
$\epsilon_1 < \PR(A=a|X=x)<1-\epsilon_1$
\end{assumption}

\begin{assumption}[2.b Positivity (censoring)] Censoring is non-deterministic, i.e. for some $\epsilon_2>0$, we have that  $\PR(N_C(t)=0|A=a, X=x) = \PR(C>t|A=a, X=x) = >\epsilon_2$ for all $t < \tau$.
\end{assumption}

\begin{assumption}[2.c Positivity (events)] Not all events deterministically occur before time $\tau$, i.e.  $\mathbb{P}(N_T(\tau \minus 1)=0|A=a, X=x)> \PR(T>\tau-1|A=a, X=x) \epsilon_3>0$
\end{assumption}

Assumptions 1.a, 1.c and 2.a are standard within the treatment effect estimation literature \cite{alaa2018limits, Shalit:16}; assumptions 1.b and 2.b are standard within the literature with survival outcomes \cite{diaz2018targeted, cui2020estimating}. Assumption 2.c is needed only if we aim to estimate $\lambda^a(t|x)$ for all $t$, otherwise it would suffice to follow a convention such as setting $\lambda^a(t|x)=1$ whenever $\mathbb{P}(N_T(\tau \minus 1)=0|A=a, X=x)=0$.

\subsection{Proof of proposition 1}
In this section we state the proof of proposition 1 and restate two lemmas from \cite{johansson2019support} which we use within the proof.

\paragraph{Notation and definitions (restated)}
For fixed $a, \tau$ and representation $\Phi: \mathcal{X} \rightarrow \mathcal{R}$, let $\mathbb{P}_0^{\Phi}$, $\mathbb{P}^{\Phi}_{a, \tau}$ and $\mathbb{P}^{w, \Phi}_{a, \tau}$ denote the baseline, observational and weighted observational distribution w.r.t. the representation $\Phi$. Define the pointwise losses
\begin{equation}
\begin{split}
    l_{h, \mathbb{Q}}(x; a, \tau) &\defeq \mathbb{E}_{Y(\tau)|x, a \sim \mathbb{Q}}[\ell(Y(\tau), h(\Phi(X)))|X=x, A=a] \\
    l_{h, \mathbb{Q}^\Phi}(\phi; a, \tau) &\defeq \mathbb{E}_{Y(\tau)|\phi, a \sim \mathbb{Q}^\Phi}[\ell(Y(\tau), h(\Phi))|\Phi=\phi, A=a]
\end{split}
\end{equation}
of (hazard) hypothesis $h: \mathcal{R} \rightarrow [0,1]$ w.r.t. distributions in covariate and representation space, respectively. 
 
Further, define the integral probability metric distance (IPM) w.r.t. a function class $\mathcal{G}$ as 
\begin{equation}
    \IPM_\mathcal{G}(\mathbb{P}, \mathbb{Q}) = \sup_{g\in \mathcal{G}}\left|\int g(x) (\mathbb{P}(x) - \mathbb{Q}(x))dx\right|
\end{equation}

Define the excess target information loss $\eta^{\ell}_\Phi(h)$ analogously to \cite{johansson2019support} as 
\begin{equation}
    \eta^{\ell}_\Phi(h) \defeq \mathbb{E}_{X\sim \mathbb{P}_0}[\xi_{\mathbb{P}^\Phi, \mathbb{P}}(X) - \xi_{\mathbb{P}^{w,\Phi}_{a, \tau}, \mathbb{P}}(X)]
\end{equation}
  with 
 \begin{equation}
\xi_{ \mathbb{Q}^\Phi, \mathbb{Q}}(x) \defeq  \ell_{h, \mathbb{Q}^\Phi}(\phi; a, \tau) - \ell_{h, \mathbb{Q}}(x; a, \tau)
 \end{equation} 
 
\paragraph{Preliminaries}
\begin{lemma}[Adapted from Lemma A.3 in \cite{johansson2019support}]\label{lemma:1}
\begin{equation*}
\mathbb{E}_{X \sim \mathbb{Q}}[\ell_{h, \mathbb{Q}}(X; a, \tau)] = \mathbb{E}_{\Phi \sim \mathbb{Q}^\Phi}[\ell_{h, \mathbb{Q}^\Phi}(\Phi; a, \tau)]
\end{equation*}
\end{lemma}
\begin{proof} This proof is adapted to our notation and setting from \cite{johansson2019support} and stated for completeness. Let $y=y(\tau)$ and $z=\Phi(x)$
 \begin{equation*}
 \begin{split}
\mathbb{E}_{\Phi \sim \mathbb{Q}^\Phi}[\ell_{h, \mathbb{Q}^\Phi}(\Phi; a, \tau)] 
&= \int_{z, y} \mathbb{Q}^\Phi(z, y) \ell(y, h(z))dz dy \\
&=\int_{z, y}  \ell(y, h(z)) \int_{x:  z = \Phi(x)} \mathbb{Q}(x, y) dxdzdy  \\
&=\int_{x, y} \mathbb{Q}(x, y) \int_z \mathbbm{1}\{z=\Phi(x)\} \ell(y, h(z)) dzdxdy \\
&= \int_{x, y} \mathbb{Q}(x, y) \ell(y, h(\Phi(x))dxdy \\
&= \mathbb{E}_{X \sim \mathbb{Q}}[\ell_{h, \mathbb{Q}}(X; a, \tau)]
\end{split}
\end{equation*}
\end{proof}

\begin{lemma}[Adapted from Lemma A.4 in \cite{johansson2019support}]\label{Lemma:2}
\begin{equation*}
    \mathbb{E}_{X \sim \mathbb{P}_0}[\ell_{h, \mathbb{P}_0}(X; a, \tau)] =  \mathbb{E}_{\Phi \sim \mathbb{P}^\Phi_0}[l_{h, \mathbb{P}^{w, \Phi}_{a, \tau}}(\Phi; a, \tau)] + \eta^l_\Phi(h)
\end{equation*}
\end{lemma}
\begin{proof} This proof is adapted to our notation and setting from \cite{johansson2019support} and stated for completeness. 
\begin{equation*}
\begin{split}
     &\mathbb{E}_{X \sim \mathbb{P}_0}[\ell_{h, \mathbb{P}_0}(X; a, \tau)] 
     = \mathbb{E}_{X \sim \mathbb{P}_0}[\mathbb{E}_{Y(\tau)|x, a \sim \mathbb{Q}}[\ell(Y(\tau), h(\Phi(X)))|X=x, A=a]]\\
    &~~~~~~~~~~~~~~~~~~~~~= \mathbb{E}_{\Phi \sim \mathbb{P}^\Phi_0}[\ell_{h, \mathbb{P}^\Phi_0}(\Phi; a, \tau)]  \text{  (by Lemma \ref{lemma:1})} \\
    &~~~~~~~~~~~~~~~~~~~~~= \mathbb{E}_{\Phi \sim \mathbb{P}^\Phi_0}[\ell_{h, \mathbb{P}^{w,\Phi}_{a, \tau}}(\Phi; a, \tau)] + \mathbb{E}_{\Phi \sim \mathbb{P}^\Phi_0}[\ell_{h, \mathbb{P}^\Phi_0}(\Phi; a, \tau)] - \mathbb{E}_{\Phi \sim \mathbb{P}^{\Phi}_0}[\ell_{h, \mathbb{P}^{w,\Phi}_{a, \tau}}(\Phi; a, \tau)] \\
   &~~~~~~~~~~~~~~~~~~~~~=  \mathbb{E}_{\Phi \sim \mathbb{P}^\Phi_0}[\ell_{h, \mathbb{P}^{w,\Phi}_{a, \tau}}(\Phi; a, \tau)] + \mathbb{E}_{\Phi \sim \mathbb{P}^\Phi_0}[\xi_{ \mathbb{P}^\Phi_0, \mathbb{P}}(X) - \xi_{ \mathbb{P}^{w,\Phi}_{a, \tau}, \mathbb{P}}(X)]\\
   &~~~~~~~~~~~~~~~~~~~~~=  \mathbb{E}_{\Phi \sim \mathbb{P}^\Phi_0}[ l_{h, \mathbb{P}^{w, \Phi}_{a, \tau}}(\Phi; a, \tau)] + \eta^l_\Phi(h)
\end{split}
\end{equation*}
where the second to last line follows as $l_{h, \mathbb{P}}(x; a, \tau)$ cancels in $\eta^l_\Phi(h)$  \end{proof}
 
\paragraph{Proof of proposition 1}
\begin{proposition}[Restated] Assume there exists a constant $C_\Phi>0$ s.t. ${C_\Phi}^{-1} \ell_{h,\mathbb{P}^{w, \Phi}_{a, \tau}}(\phi, a, \tau) \in \mathcal{G}$ for some family of functions $\mathcal{G}$. Then we have that 
\begin{equation}\label{eq:mainbound_apendix}
   \underbrace{\mathbb{E}_{X \sim \mathbb{P}_0}[\ell_{h, \mathbb{P}}(X; a, \tau)]}_{\text{Target Risk}} \leq  \underbrace{\mathbb{E}_{X \sim \mathbb{P}_{a, \tau}}[w_{a, \tau}(X)\ell_{h, \mathbb{P}}(X; a, \tau)]}_{\text{Weighted observational risk}} + C_\Phi \underbrace{\IPM_G(\mathbb{P}^{\Phi}_0, \mathbb{P}^{w, \Phi}_{a, \tau})}_{\text{Distance in } \Phi \text{-space}} + \underbrace{\eta^l_\Phi(h)}_{\text{Info loss}}
\end{equation}
\end{proposition} 
\begin{proof}
By Lemma \ref{Lemma:2},  
\begin{equation*}
    \mathbb{E}_{X \sim \mathbb{P}_0}[\ell_{h, \mathbb{P}_0}(X; a, \tau)] =  \mathbb{E}_{\Phi \sim \mathbb{P}^\Phi_0}[l_{h, \mathbb{P}^{w, \Phi}_{a, \tau}}(\Phi; a, \tau)] + \eta^l_\Phi(h)
\end{equation*}

Further, 
\begin{equation*}
\begin{split}
    &\mathbb{E}_{\Phi \sim \mathbb{P}^\Phi_0}[ l_{h, \mathbb{P}^{w, \Phi}_{a, \tau}}(\Phi; a, \tau)] -  \mathbb{E}_{\Phi \sim \mathbb{P}^{w, \Phi}_{a, \tau}}[ l_{h, \mathbb{P}^{w, \Phi}_{a, \tau}}(\Phi; a, \tau)] = \int_{\phi} \ell_{h, \mathbb{P}^{w, \Phi}_{a, \tau}}(\phi; a, \tau) (\mathbb{P}^\Phi_0(\phi) - \mathbb{P}^{w, \Phi}_{a, \tau}(\phi))d\phi \\
    &\qquad\qquad\qquad\qquad~~~~~~~~~~~~~~~~~~~~~~~~~~~~~~~~~~~~~~~~~~~~~~~= C_\Phi\int_{\phi} \frac{\ell_{h, \mathbb{P}^{w, \Phi}_{a, \tau}}(\phi; a, \tau)}{C_\Phi} (\mathbb{P}^\Phi_0(\phi) - \mathbb{P}^{w, \Phi}_{a, \tau}(\phi))d\phi \\
     &\qquad\qquad\qquad\qquad~~~~~~~~~~~~~~~~~~~~~~~~~~~~~~~~~~~~~~~~~~~~~~~\leq C_\Phi \sup_{g \in \mathcal{G}} \left|\int_{\phi} g(\phi) (\mathbb{P}^\Phi_0(\phi) - \mathbb{P}^{w, \Phi}_{a, \tau}(\phi))d\phi\right| \\
    &\qquad\qquad\qquad\qquad~~~~~~~~~~~~~~~~~~~~~~~~~~~~~~~~~~~~~~~~~~~~~~~= C_\Phi \IPM_G(\mathbb{P}^{\Phi}_0, \mathbb{P}^{w, \Phi}_{a, \tau})
\end{split}
\end{equation*}

Thus 
\begin{equation*}
\begin{split}
    \mathbb{E}_{X \sim \mathbb{P}_0}[\ell_{h, \mathbb{P}_0}(X; a, \tau)] &\leq \mathbb{E}_{\Phi \sim \mathbb{P}^{w, \Phi}_{a, \tau}}[ l_{h, \mathbb{P}^{w, \Phi}_{a, \tau}}(\Phi; a, \tau)] + C_\Phi \IPM_G(\mathbb{P}^{\Phi}_0, \mathbb{P}^{w, \Phi}_{a, \tau}) + \eta^l_\Phi(h) \\
    &=\mathbb{E}_{X \sim \mathbb{P}^{w}_{a, \tau}}[ \ell_{h, \mathbb{P}}(X; a, \tau)] + C_\Phi \IPM_G(\mathbb{P}^{\Phi}_0, \mathbb{P}^{w, \Phi}_{a, \tau}) + \eta^l_\Phi(h)
    \end{split}
\end{equation*}
where the last line follows by Lemma \ref{lemma:1} and the unconfoundedness and censoring at random assumptions, by which $ \ell_{h, \mathbb{P}}(X; a, \tau)= \ell_{h, \mathbb{P}^w_{a, \tau}}(X; a, \tau)$

\end{proof}

\section{Implementation} \newnumbering
We discuss implementation of \proposed~ and baselines in turn below. The source code for \proposed~is available in {\url{https://github.com/chl8856/survITE}}. Throughout the experiments, training \proposed~and its variants takes approximately 30 minutes to 1 hour on a single GPU machine\footnote{The specification of the machine is: CPU – Intel Core i7-8700K, GPU – NVIDIA GeForce GTX 1080Ti,and RAM – 64GB DDR4.}.

\subsection{\proposed}
Throughout the experiments, we implement \proposed~utilizing 3-layer fully-connected network (FC-Net) with 100 nodes in each layer for the representation estimator $\Phi$, and 2-layer FC-Net with 100 nodes in each layer for each hypothesis estimator $h_{a, t}$, respectively. The parameters $(\theta_\Phi, \theta_h)$ are initialized by Xavier initialization \cite{Xavier:10} and optimized via Adam optimizer \cite{Adam:14} with learning rate of 0.001 and dropout probability of 0.3. We choose the balancing coefficient $\beta$ within a set of possible candidates $\mathcal{B} = \{1., 0.1, 0.01, 0.001, 0.0001\}$ utilizing a grid search. More specifically, we select the highest value in $\mathcal{B}$ that does not compromise its discriminative performance (i.e. C-Index) based on the validation set (i.e., 20\% of the training set) to guarantee that the learned representation is balanced as much as possible to adjust for the covariate shift while being informative about the survival predictions. The effect of the balancing coefficient is further investigated in Section \ref{sec:sensitivity_analysis}.

{
\textbf{Finite-Sample Wasserstein Distance.}
For the finite sample approximation of the Wasserstein distance, we use Algorithm \ref{alg:wass} with the entropic regularization strength set to $\lambda=10$ and the number of Sinkhorn iterations set to $S=10$ following the implementation in \cite{cuturi2014wass,assaad2021balancingweights}. Thus, given two sets of samples $\mathcal{B}_{0}, \mathcal{B}_{1}$ based on the treatment-group time-step combinations, we can compute $Wass\big( \{\Phi(\xv_{i}) \}_{i\in \mathcal{B}_{0}} , \{\Phi(\xv_{i}) \}_{i\in \mathcal{B}_{1}}  \big)$ based on Algorithm \ref{alg:wass}.
}
\begin{center}
    \begin{minipage}{0.75\linewidth}
        \begin{algorithm}[H]
        	\caption{Pseudo-code for Finite Sample Wasserstein Distance}
        	\label{alg:wass}
        	\begin{algorithmic}
        	    \STATE {\bfseries Input:} Set $\mathcal{B}_{0}$, $\mathcal{B}_{1}$, entorpic regularization parameter $\lambda \in \mathbb{R}$, the number of Sinkhorn iterations $S$, representation $\theta_{\Phi}$
        	    \STATE $n_{1} = |\mathcal{B}_{1}|$ and $n_{0} = |\mathcal{B}_{0}|$
        	    \STATE $a = \frac{1}{n_{1}} \mathbf{1} \in \mathbb{R}^{n_{1}}$ and $b = \frac{1}{n_{0}} \mathbf{1} \in \mathbb{R}^{n_{0}}$
        	    \STATE $M^{(i,j)} = \| \Phi(\xv_{i}) - \Phi(\xv_{j})  \|_{2}$ $\forall{i}\in \mathcal{B}_{1}$, $\forall{j}\in \mathcal{B}_{0}$
        	    \STATE $K = \exp(-\lambda M)$
        	    \STATE $\tilde{K} = \text{diag}(a^{-1}) K$
        	    \STATE Initialize $u = a$
        	    \FOR{$s=1,\cdots, S$}
        	        \STATE $u = 1. / (\tilde{K} (b./(K^{T}u)) )$
        		\ENDFOR
        		\STATE $v=b./(K^{T}u).$
        		\STATE $T^{*}_{\lambda} = \text{diag}(u) K \text{diag}(v)$
        		\STATE {\bfseries Output:} $Wass\big( \{\Phi(\xv_{i}) \}_{i\in \mathcal{B}_{0}} , \{\Phi(\xv_{i}) \}_{i\in \mathcal{B}_{1}}  \big) \approx \sum_{i,j} T^{* (i,j)}_{\lambda} M^{(i,j)}$
        	\end{algorithmic}
        \end{algorithm}
    \end{minipage}
\end{center}

\textbf{Smoothing and Parameter Sharing.}
Employing a separate FC-Net at each time step provides sufficient capacity to estimate the hazard function accurately. However, this can be computationally burdensome as the number of parameters linearly increases with the number time steps considered in the study, and may result in having hypothesis estimators overfitted at the later time steps due to the scarcity of samples at those time steps. To avoid such issues, one can employ coarser time intervals for discritization or non-uniform time intervals (as in our experiments on the Twins dataset) such that finer time intervals are used in the earlier time steps and coarser time intervals are used in the later time steps to guarantee a sufficient amount of samples for training each hypothesis network. In addition to these immediate solutions, one could consider two different remedies that slightly change our model design:

\begin{itemize}[leftmargin=5.5mm]
    \item \textbf{Smoothing regularization}: We introduce an auxiliary regularization term that smooths the hazard estimators across time steps for each treatment group. This encourages the hazard estimators not to deviated too much from those at adjacent time steps. 
    Formally, the smoothing regularization is given by
    \begin{equation} \nonumber
        \loss_{smoothing}(\theta_{h}) = \sum_{a\in\{0,1\}} \sum_{t=1}^{\tmax} \| \theta_{h_{a,t}} - \theta_{h_{a,t
        \minus 1}}\|_{2}^{2}.
    \end{equation}
    \item \textbf{Parameter Sharing}: Instead of employing a separate FC-Net for each time step, we implement a single FC-Net for each treatment group that is shared throughout the time steps i.e., $t \in \Tc$, taking both the representation $\Phi(x)$ and $t$ as input to the network. Formally, the hazard function is defined as $h_{a}: \mathcal{R} \times \Tc \rightarrow [0,1]$.
\end{itemize}
We present experimental results using these approaches in Section \ref{sec:smoothing}.

\subsection{Details of Baselines}
We compared \proposed~with baselines ranging from commonly used survival methods to the state-of-the-art HTE methods based on deep neural networks. The details of how we implemented the benchmarks are described as the following:
\begin{itemize}[leftmargin=5.5mm]
    \item \textbf{Cox}\footnote{Python package \texttt{scikit-survival} \cite{sksurv} \label{footnote:sksurv}} \cite{Cox:72}, \textbf{RSF}\footref{footnote:sksurv} \cite{Ishwaran:08}, {and \textbf{DeepHit}\footnote{\url{https://github.com/chl8856/DeepHit}}}: When there are treatments, we use these models in a two-model (T-learner) approach by training a separate model using samples in the treated ($A=1$) and controlled ($A=0$) groups, respectively. For Cox, we set the coefficient for ridge regression penalty as $\alpha=0.001$. For RSF, we use the default hyper-parameter setting (i.e., $\text{\textit{n\_estimators}}=100$ using a survival tree as the baseline estimator and $\text{\textit{min\_samples\_leaf}}=3$ without maximum depth restriction). {For DeepHit, we use utilize the 3-layer FC-Net with 100 nodes in each layer. We choose the DeepHit's hyper-parameters $\alpha, \sigma$ from a set of possible candidates $\{0.001, 0.01, 0.1, 1, 10\}$ and $\{0.01, 0.1, 1. 10\}$, respectively.}
    \item \textbf{LR-sep}: We utilize the long data format as described in Section 2 of the manuscript and train a separate logistic regression model\footnote{Python package \texttt{scikit-learn}} at each time step $t\in\Tc$ to solve the hazard classification problem utilizing only ``at-risk'' samples whose time-to-event/censoring is at or after $t$. Formally, the logistic regression models are trained based on the log-loss in \eqref{eq:loss_longformat}. When there are treatments, we use LR-sep in a two-model (T-learner) approach by training a separate model using samples in the treated ($A=1$) and controlled ($A=0$) groups, respectively. 
    \item\textbf{CSA}\footnote{\url{https://github.com/paidamoyo/counterfactual_survival_analysis}} \cite{chapfuwa2021enabling}: We use the CSA-INFO model of \cite{chapfuwa2021enabling}, where we use its generative capabilities to approximate target quantities via monte-carlo sampling. We use the code and specifications provided by the authors, in particular we use a hidden dimension of 100, set the imbalance penalty $\alpha=100$ and train for 300 epochs. To create monte carlo approximations, we sample 1000 times from the model for each observation in the test set.  
    \item \textbf{\proposed~(CFR-1)} and \textbf{\proposed~(CFR-2)}: We consider two variants of SurvITE by replacing our $\loss_{ipm}(\theta_{\phi})$ with a balancing term based on the CFRNet\footnote{\url{https://github.com/clinicalml/cfrnet}} proposed in \cite{Shalit:16}: 
    \begin{itemize}[leftmargin=5.5mm]
        \item \textbf{SurvITE (CFR-1)} creates a representation balancing treatment groups at baseline only which is formally given as:
        \begin{equation}
            \loss_{ipm}(\theta_{\phi}) = Wass\big( \{ {\Phi}(x_{i}) \}_{i:a_{i}=1}, \{  {\Phi}(x_{i}) \}_{i:a_{i}=0} \big)
        \end{equation}
        \item \textbf{SurvITE (CFR-2)} creates a representation optimizing for balance of treatment groups \textit{at each time step}
        \begin{equation}
            \loss_{ipm}(\theta_{\phi}) = \sum_{t=1}^{\tmax} Wass\big( \{  {\Phi}(x_{i}) \}_{i:\tilde{\tau}_{i} \geq t, a_{i}=1}, \{  {\Phi}(x_{i}) \}_{i:\tilde{\tau}_{i} \geq t, a_{i}=0} \big)
        \end{equation}
    \end{itemize}
    Note that, in both variants, there is no balancing towards $\mathbb{P}_0$. We implement SurvITE (CFR-1) and SurvITE (CFR-2) with the same network architecture and hyper-parameters with those of SurvITE.
\end{itemize}

\section{Dataset Description and Experimental Setup} \newnumbering

\subsection{Synthetic Experiments}
In this section, we present some illustrations of the properties of the synthetic DGPs. Recall that $\lambda_w(\tau|x)$ is the same across all settings, therefore we focus here on S3 to analyze the interplay of selection bias and event processes. In Fig. \ref{fig:event_times}, we present histograms of event times in S3 for different degrees of selection bias. Note that there is a positive treatment effect (treatment reduces event probabilities) encoded in our DGP; this is clearly visible in the left panel without selection bias as consistently more events occur for control than for treated group. As we add selection bias, it seems that treatment has a \textit{negative effect} on survival after time 10, as more events occur in the treatment group. This correlation is spurious: as $X_2$ linearly increases event hazard, and treatment is selected based on $X_2$ (rightmost panel) or based on a variable correlated with it (middle panel) it seems \textit{as if} treatment increases mortality. Note that this is not the case for $t<10$ because $X_1$ enters $\lambda_w(\tau|x)$ in squared form. 

\begin{figure}[h]
    \centering
    \includegraphics[width=0.8\textwidth]{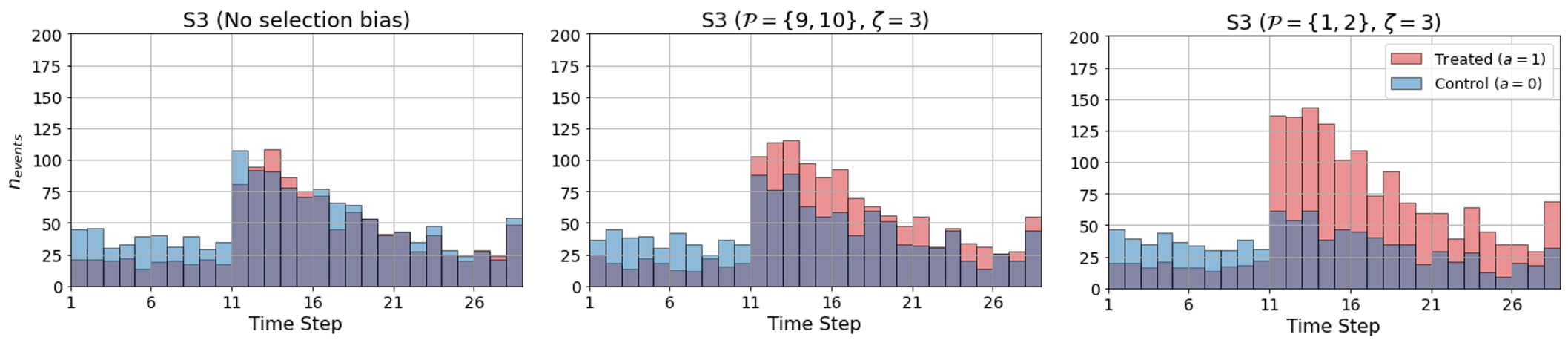}
    \caption{Histograms of event times in S3 for different degrees of selection bias; no selection (left), no overlapping selection covariates (middle) and full overlap (right)}
    \label{fig:event_times}
\end{figure}
In Figure \ref{fig:cov_shift} we further analyze the interplay between event-induced covariate shift and selection bias by considering the distribution of $X_1$ in the at-risk population over time. As $-X_1^2$ appears in the hazard, individuals with small magnitude of $X_1$ have lower probability of survival -- this becomes visible for $\zeta=0$ as the at-risk histogram flattens out over time. Because $X_1$ enters $e(x)$ linearly, when we add selection bias ($\zeta>0$), we observe that the populations not only differ already at baseline, but that the difference appears to become more extreme over time -- this is precisely because the overlapping parts of the population ($|X_1|$ small) have larger event probability, so that the event-induced shift further amplifies the selection boas over time.

\begin{figure}[h]
    \centering
    \includegraphics[width=0.6\textwidth]{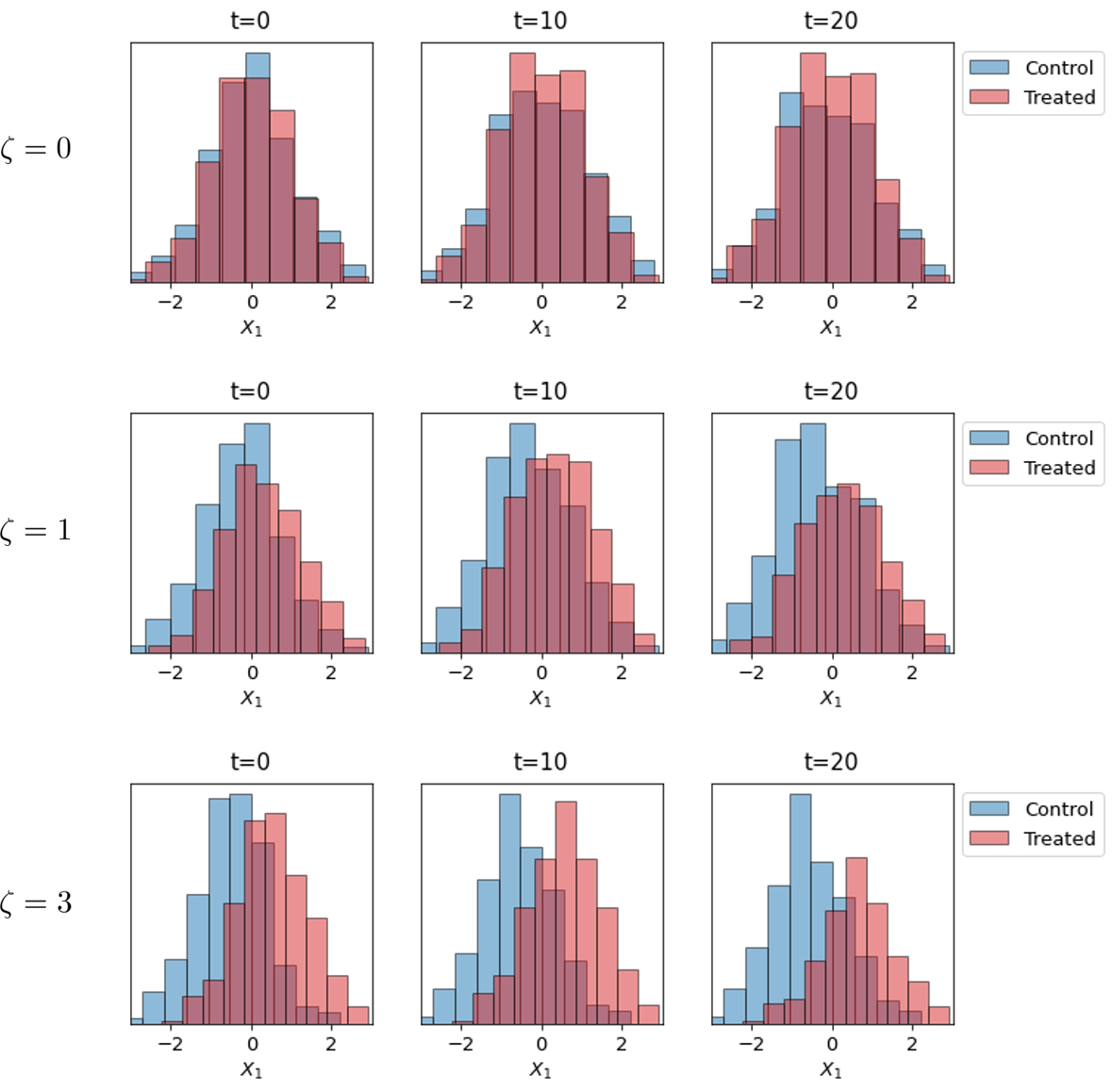}
    \caption{Histograms of $X_1$ in the at-risk population by time (left to right) and $\zeta$ (top to bottom) for S3 with $\mathcal{P}=\{1, 2\}$, highlighting covariate shift due to selection bias (across $\zeta$) and occurred events (across $t$)}
    \label{fig:cov_shift}
\end{figure}

\subsection{Semi-Synthetic Experiments: Twins}
This dataset is derived from all births in the USA between 1989--1991 \cite{Almond:05} where we only focus on the twins. 
We artificially create a binary treatment such that $a = 1 (a = 0)$ denotes being born the heavier (lighter). The outcome of interest is the time-to-mortality (in days) of each of the twins in their first year, thus administratively censored at $t=365$. Since we have records for both twins, we treat their time-to-event outcomes as two potential outcomes, i.e., $\tilde{\tau}(1)$ and $\tilde{\tau}(0)$, with respect to the treatment assignment of being born heavier. 
As previously used in \cite{yoon2018ganite, louizos2017causal}, we obtained 30 features (which has 39 feature dimension after one-hot encoding on categorical features)  for each twin relating to the parents, the pregnancy, and the birth (e.g., marital status, race, residence, number of previous births, pregnancy risk factors, quality of care during pregnancy, and number of gestation weeks prior to birth). We only chose twins weighing less than 2kg and without missing features. To create an observational time-to-event dataset, we selectively observed one of the two twins \textbf{(no censoring)} with selection bias and \textbf{(censoring)} with both selection bias and censoring bias as follows: the treatment assignment is given by $a|x \sim \texttt{Bern}(\sigma(w_{1}^{\top}x + e))$ where $w\sim \texttt{Uniform}(-0.1, 0.1)^{39\times1})$ and $e\sim \Norm(0, 1^{2})$, and the time-to-censoring is given by $C\sim \texttt{Exp}(100 \cdot \sigma(w_{2}^{\top} x))$ where $w_{2}\sim \Norm(0, 1^{2})$. 

For continuous-time models (i.e., Cox, RSF, and CSA), we use the original time resolution in days without discarding any information. For discrete-time models (i.e., LR-sep, \proposed~, and variants of \proposed), we use a non-uniform discretization -- i.e. resolution of days in the first 30 days and months after the first 30 days -- because most of the events are concentrated in the first 30 days (approximately 87\% of the events occur within that period).

\subsection{Performance Metrics}
Once \proposed~(or SurvIHE) is trained, we can simply estimate the (treatment-specific) survival function based on the estimated hazard functions as the following:
\begin{equation}
    \hat{S}^{a}(\tau|x) = \prod_{t\leq \tau} \big( 1 - h_{a,t}(\Phi(x)) \big)~~~~~~~\text{for}~~a\in\{0,1\}.
\end{equation}

\textbf{Heterogeneous Treatment Effects.} 
For synthetic experiments where we have the ground-truth treatment-specific survival functions i.e., $S^{1}(\tau|x)$ and $S^{0}(\tau|x)$, we evaluate $HTE_{surv}(\tau|x) = S^{1}(\tau|x) - S^{0}(\tau|x)$ and $HTE_{rmst}(x; L) = 
\sum_{t_{k} \leq L} \big(S^{1}(t_{k}|x) - S^{0}(t_{k}|x) \big) \cdot (t_{k}- t_{k{-}1})$ in terms of the averaged root mean squared error (RMSE) of the estimation:
\begin{align}
\epsilon_{HTE_{surv}}(t) &= \sqrt{\frac{1}{n}\sum_{i=1}^{n} \big(HTE_{surv}(t|x_{i}) - \widehat{HTE}_{surv}(t|x_{i})\big)^{2}}, \\
\epsilon_{HTE_{rmst}}(L) &= \sqrt{\frac{1}{n}\sum_{i=1}^{n} \big(HTE_{rmst}(x_{i}; L) - \widehat{HTE}_{rmst}(x_{i}; L)\big)^{2}}. \label{eq:rmse_hte_rmst}
\end{align}
Here $\widehat{HTE}_{surv}(t|x) = \hat{S}^{1}(\tau|x) - \hat{S}^{0}(\tau|x)$ and $\widehat{HTE}_{rmst}(x; L) = \sum_{t_{k} \leq L} \big(\hat{S}^{1}(t_{k}|x) - \hat{S}^{0}(t_{k}|x) \big) \cdot (t_{k}- t_{k{-}1})$ where $(t_{k} - t_{k-1})$ may vary depending on how the continuous time is discretized (e.g., non-uniform time intervals for the Twins dataset).

For semi-synthetic experiments where we have the ground-truth treatment-specific time-to-event outcomes but not the treatment-specific survival functions, we only report $\epsilon_{HTE_{rmst}}(L)$ in \eqref{eq:rmse_hte_rmst} where the ground-truth $HTE_{rmst}(x; L)$ is defined in terms of the ground-truth time-to-event outcomes, i.e., $HTE_{rmst}(x; L) = (\min(T(1), L) - \min(T(0),L))$ where $T(1)$ and $T(0)$ are the time-to-event given $a=1$ and $a=0$, respectively.

\textbf{(Treatment-Specific) Survival Functions.}  For evaluating the estimation performance of the (treatment-specific) survival functions, we evaluate the averaged RMSE of these estimations as the following:
\begin{equation} \label{eq:c_index}
    \epsilon_{S^{a}}(t) =\sqrt{ \frac{1}{n}\sum_{i=1}^{n} \big(S^{a}(t|x_{i}) - \hat{S}^{a}(t|x_{i})\big)^{2}}.
\end{equation}

\textbf{Discriminative Performance.} 
For assessing the survival predictions of all the survival models with respect to how well the predictions discriminate among individual risks, we use the concordance index (C-Index) \cite{uno2011cindex}: 
\begin{equation}
    \text{C-Index}(t) = \PR\big(\hat{S}(t|x_{i}) < \hat{S}(t|x_{j}) \big| \tilde{\tau}_{i} < \tilde{\tau}_{j}, \tilde{\tau}_{i} \leq t, \delta_{i} = 1 \big)
\end{equation}
where $\hat{S}(t|x) = a\cdot\hat{S}^{1}(t|x) + (1-a)\cdot \hat{S}^{0}(t|x)$ is the survival prediction given treatment $a$. The resulting C-Index in \eqref{eq:c_index} tells us how well the given survival model discriminates the individual risks among the events that occur before or at time $t$.

\section{Additional Experiments}
\newnumbering

\subsection{Additional Results on the Synthetic Experiments}
In this subsection, we report the additional results on the synthetic experiments that were not provided in the manuscript due space constraints. 

{
In Table \ref{Table:DeepHit_added}, we report the performance comparison using DeepHit with respect to the estimations on both $S^{0}(t|x)$ and $HTE_{surv}(t|x)$ for the synthetic settings S3 and S4 with $\zeta=3$ and $\mathcal{P}=\{9,10\}$ at $L=10$.  
We observe that DeepHit performs worse than the \proposed~architecture without IPM term, indicating that our model architecture alone is better suited for estimation of treatment-specific survival functions (note that \cite{Changhee:AAAI18} focused mainly on discriminative (predictive) performance, and not on the estimation of the survival function itself). Therefore, upon addition of the IPM-terms, the performance gap between \proposed~and DeepHit only becomes larger.}
\begin{table}[t!]
	\caption{RMSE on estimation of $S^{0}(t|x)$ and $\text{HTE}_{surv}(x)$  (mean $\pm$ 95\%-CI) for the synthetic settings S3 and S4 with $\zeta=3$ and $\mathcal{P}=\{9,10\}$ at $L=10$.} \label{Table:DeepHit_added}
	\begin{center}
	\footnotesize
        \begin{tabular}{c c c c c}
        \toprule
        \multirow{2}{*}{\textbf{Methods}}
        &\multicolumn{2}{c}{RMSE on $S^{0}(t|x)$}
        &\multicolumn{2}{c}{RMSE on $HTE_{surv}(t|x)$} \\
        &S3&S4&S3&S4\\\midrule
        Cox	                &0.127$\pm$0.002 &0.127$\pm$0.001 &0.101$\pm$0.004 &0.099$\pm$0.004 \\
        RSF	                &0.074$\pm$0.005 &0.079$\pm$0.005 &0.081$\pm$0.003 &0.084$\pm$0.004 \\
        LR-sep          	&0.112$\pm$0.002 &0.115$\pm$0.006 &0.096$\pm$0.003 &0.099$\pm$0.005 \\
        DeepHit         	&0.095$\pm$0.012 &0.087$\pm$0.003 &0.107$\pm$0.014 &0.095$\pm$0.007 \\
        CSA	                &0.155$\pm$0.005 &0.147$\pm$0.001 &0.176$\pm$0.025 &0.148$\pm$0.011 \\ \midrule
        \proposed~(no IPM)	&0.086$\pm$0.008 &0.088$\pm$0.011 &0.071$\pm$0.011 &0.079$\pm$0.012 \\
        \proposed~(CFR-1)	&0.084$\pm$0.003 &0.097$\pm$0.009 &0.068$\pm$0.009 &0.083$\pm$0.005 \\
        \proposed~(CFR-2)	&0.059$\pm$0.003 &0.085$\pm$0.020 &0.061$\pm$0.009 &0.073$\pm$0.011 \\
        \proposed         	&\textBF{0.055$\pm$0.007} &\textBF{0.063$\pm$0.010} &\textBF{0.060$\pm$0.009} &\textBF{0.063$\pm$0.004} \\
        \bottomrule
        \end{tabular}
	\end{center}
\end{table}

\begin{figure*}[!h]
    \centering
    \includegraphics[width=\textwidth]{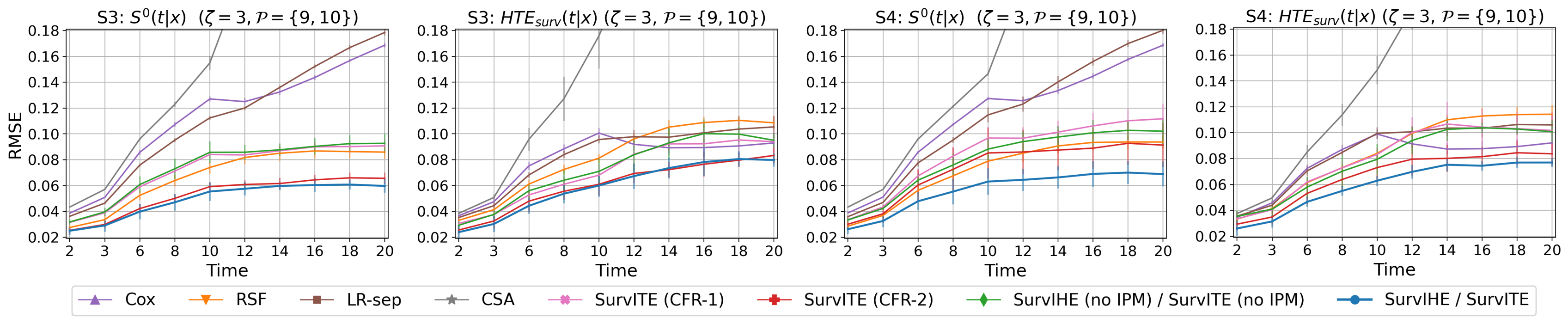}
	\caption{RMSE of estimating the treatment-specific survival function $S^{0}(t|x)$ and the treatment effect $HTE_{surv}(t|x)$ for different time steps across synthetic settings (Lower is better). Averaged across 5 runs; the error bar indicates 95\%-confidence interval. } \label{fig:survivalfunc_errorbar} 
\end{figure*}
\begin{figure*}[!h]
    \centering
    \includegraphics[width=\textwidth]{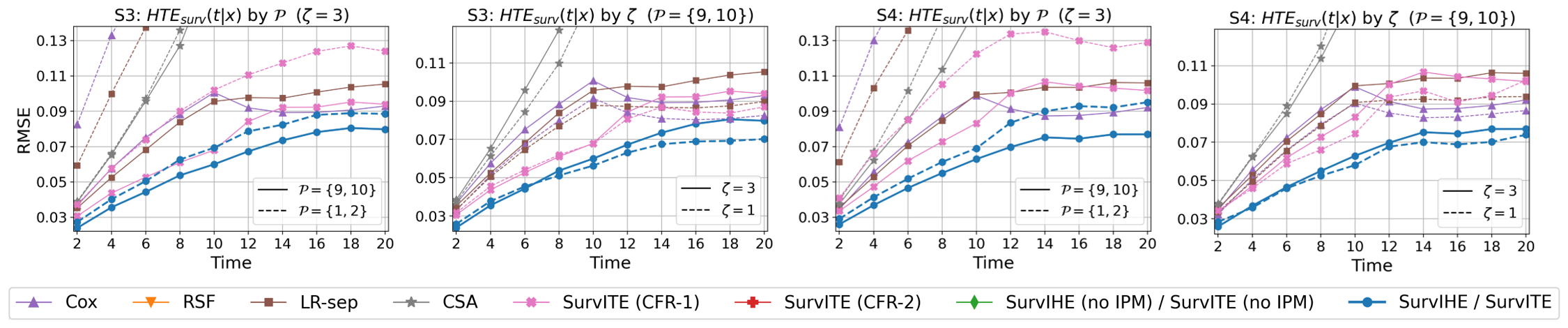}
	\caption{RMSE of estimating the treatment effect $HTE_{surv}(t|x)$ for different time steps across synthetic settings (Lower is better) for methods not presented in the main text. Averaged across 5 runs.} \label{fig:survivalfunc_othermethods} 
\end{figure*}
\begin{figure*}[!h]
    \centering
    \includegraphics[width=\textwidth]{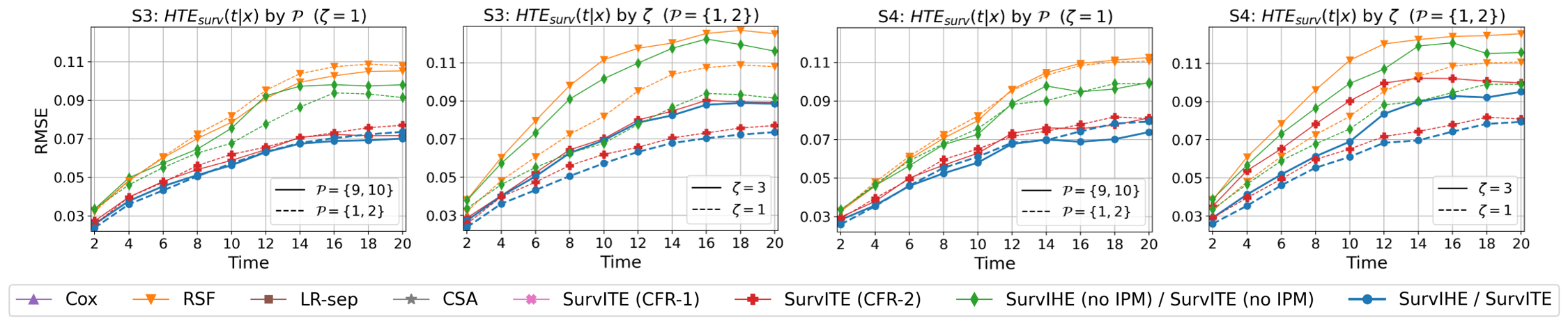}
	\caption{RMSE of estimating the treatment effect $HTE_{surv}(t|x)$ for different time steps across synthetic settings (Lower is better) for settings not presented in the main text. Averaged across 5 runs.} \label{fig:survivalfunc_others} 
\end{figure*}

Figure \ref{fig:survivalfunc_errorbar} shows the performance of estimations on $S^{0}(t|x)$ and $HTE_{surv}(t|x)$ with error bars (omitted in the main text for readability), Figure \ref{fig:survivalfunc_othermethods} shows the performance of $HTE_{surv}(t|x)$ estimation for survival methods that were not presented in the main text to ensure readability, and Figure \ref{fig:survivalfunc_others} shows the performance of $HTE_{surv}(t|x)$ estimation for synthetic scenarios (combinations of $\mathcal{P}$ and $\zeta$) not provided in the main text due to space constraints. In all cases, we observe that \proposed~(/SurvIHE) outperforms all other methods.

\begin{figure*}[!h]
    \centering
    \includegraphics[width=\textwidth]{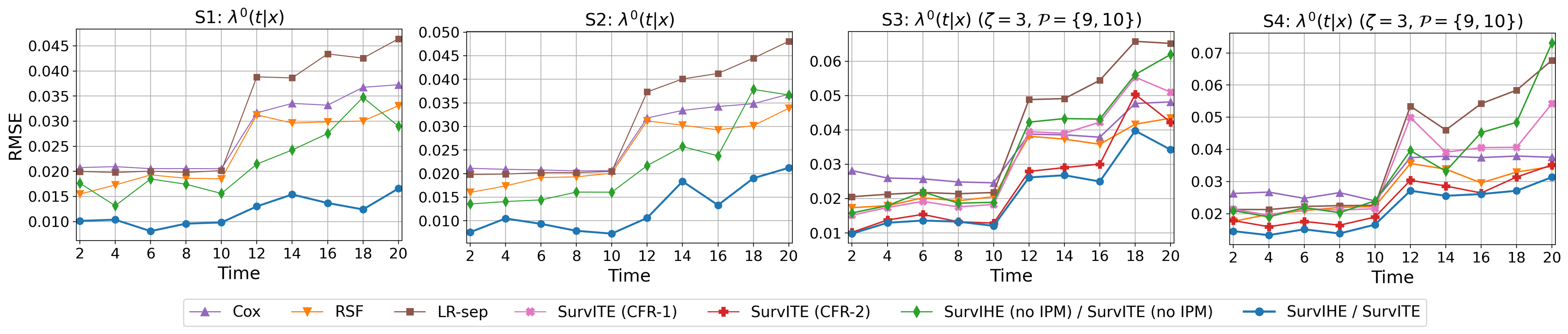}
	\caption{RMSE of estimating the treatment-specific hazard function $\lambda^{0}(t|x)$ for different time steps across synthetic settings (Lower is better). Averaged across 5 runs.} \label{fig:rmsehaz} 
\end{figure*}

In Figure \ref{fig:rmsehaz} we present the RMSE of estimating the \textit{hazard} function instead of the \textit{survival} function as in the main text\footnote{Note that we excluded CSA-INFO from this plot, as it only outputs real-valued time-to-event predictions, which makes it unclear how to best use it to directly estimate hazard functions. }. The results for hazard estimation largely mimic the ones presented in the main text; in particular, we observe that \proposed~(/SurvIHE) performs best throughout. Noteably, the gaps in performance across all methods at later time-steps appear somewhat smaller for hazard than for survival functions; this is expected as the errors on hazards accumulate when the survival function is computed from them. 

\begin{figure*}[!h]
    \centering
    \includegraphics[width=\textwidth]{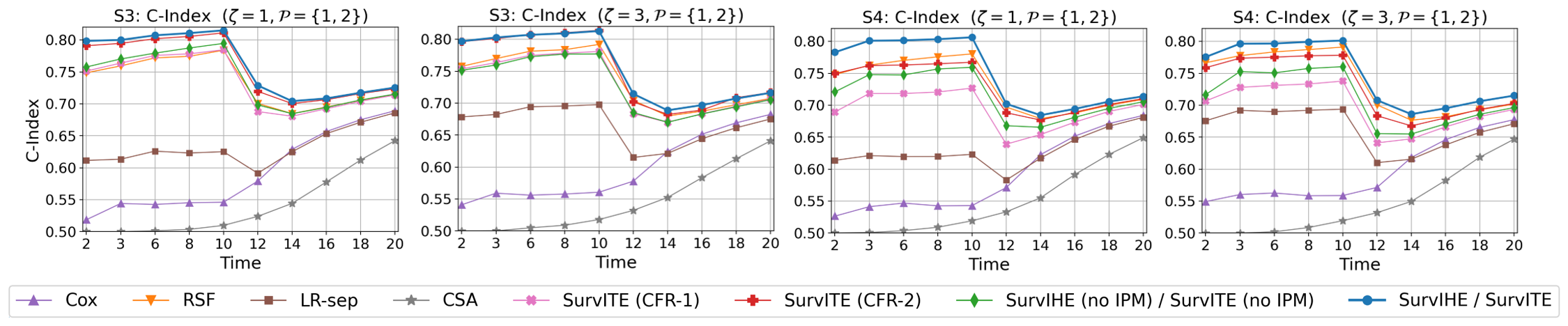}
    \caption{C-Index for different time steps for S3 and S4 with $\zeta=3$ and $\mathcal{P}=\{1,2\}$ (Higher is better). Averaged across 5 runs.} \label{fig:cindex} 
\end{figure*}

In addition, in Figure \ref{fig:cindex}, we report the discriminative performance of the various survival models for synthetic scenarios S3 and S4 with $\zeta=3$ and $\mathcal{P}=\{1,2\}$ across different time steps. We evaluate the discriminative performance in terms of the C-index defined in \eqref{eq:c_index}. The figure shows that \proposed~performs the best also in terms of discriminative performance throughout different scenarios and different time steps due to the accurate estimation of the treatment-specific survival functions.

\subsection{Sensitivity Analysis}\label{sec:sensitivity_analysis}
In this subsection, we investigate the effect of the balancing coefficient $\beta$ in Figure \ref{fig:betas} on the estimation performance of the HTE, and the discriminative performance of the survival predictions. As expected, Figure \ref{fig:betas} shows that \proposed~with a proper amount of IPM regularization improves the treatment effect estimation: imposing too much regularization will make the representation estimator unable to maintain important information for estimating treatment-specific hazard functions while setting regularization too low will not balance the representation from the different sources of covariate shift. Similarly, if the representation is balanced too much, it will lose discriminative power which will eventually make the trained model less useful. In this context, due to the absence of counterfactual information in practice, we propose to select the balancing coefficient $\beta$ by increasing the value starting from the lowest value in the set of possible candidates $\{1., 0.1, 0.01, 0.001, 0.0001\}$ as long as the method maintains good discriminative performance on the validation set (and stop when discriminative performance deteriorates). In our experiments, we choose $\beta=0.001$, which is the largest value that provides good discriminative performance based on the validation set (see the right hand panel in Figure \ref{fig:betas}).
\begin{figure*}[!h]
    \centering
    \includegraphics[width=0.72\textwidth]{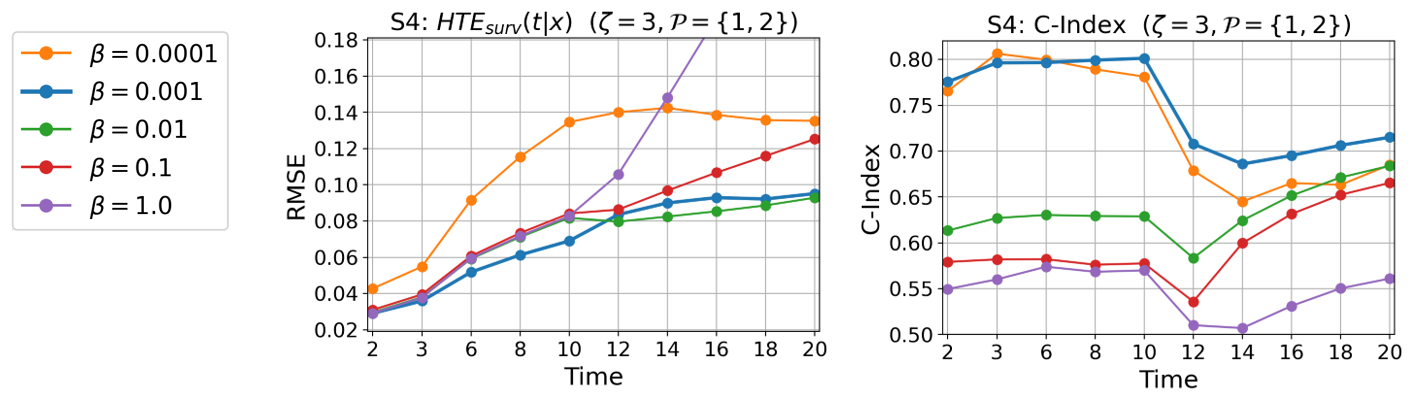}
	\caption{RMSE of estimating the treatment effect $HTE_{surv}(t|x)$ and C-index of the survival predictions with different $\beta$ on S4 with $\zeta=3$ and $\mathcal{P}=\{1,2\}$. Averaged across 5 runs.} \label{fig:betas} 
\end{figure*}

\subsection{\proposed~Variants with Smoothing and Sharing Parameters} \label{sec:smoothing}
In this subsection, we further investigate \proposed~variants with techniques that can address the practical issue of potentially having too many separate hypothesis estimators for large $t_{max}$. Figure \ref{fig:parameters} compares the estimation performance of the treatment-specific survival functions, estimation performance of the HTE, and the discriminative performance of survival predictions. When the models are trained with a sufficient number of samples (here: 5000 training samples), the smoothing regularization maintains a very similar performance in terms of estimating the treatment-specific survival functions and the HTE while sacrificing its discriminative performance at early time step. Sharing the parameters of hypothesis estimator across different time steps suffers more performance loss (nonetheless, it still provides reasonable performance) as the flexibility of the network is more restricted. 
On the other hand, when the models are trained with a smaller number of samples (here: 2000 training samples), the smoothing and sharing the network parameters play significant role in improving the estimation performance.

\begin{figure*}[h]
	\centering
	\begin{subfigure}[t]{\textwidth}
		\centering
        \includegraphics[width=0.99\textwidth]{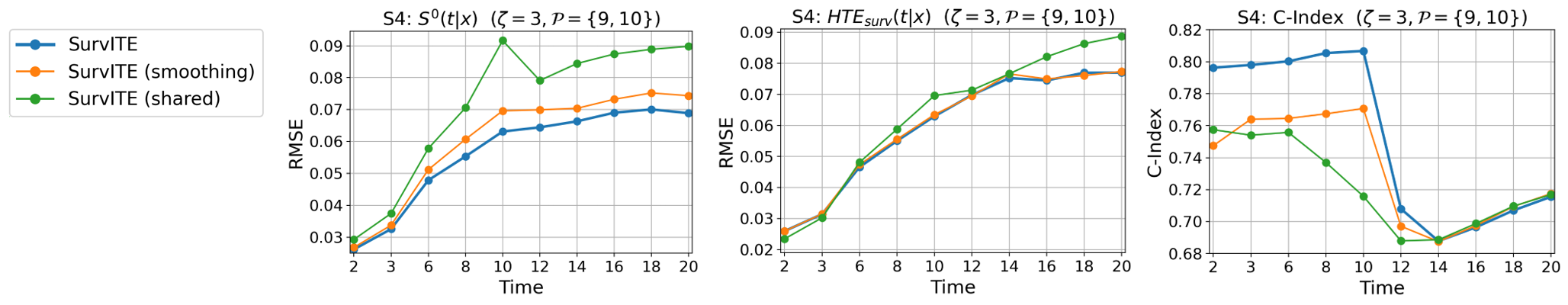}
		\caption{Training samples: 5000}
	\end{subfigure}\hfill
	\begin{subfigure}[t]{\textwidth}
		\centering
        \includegraphics[width=0.99\textwidth]{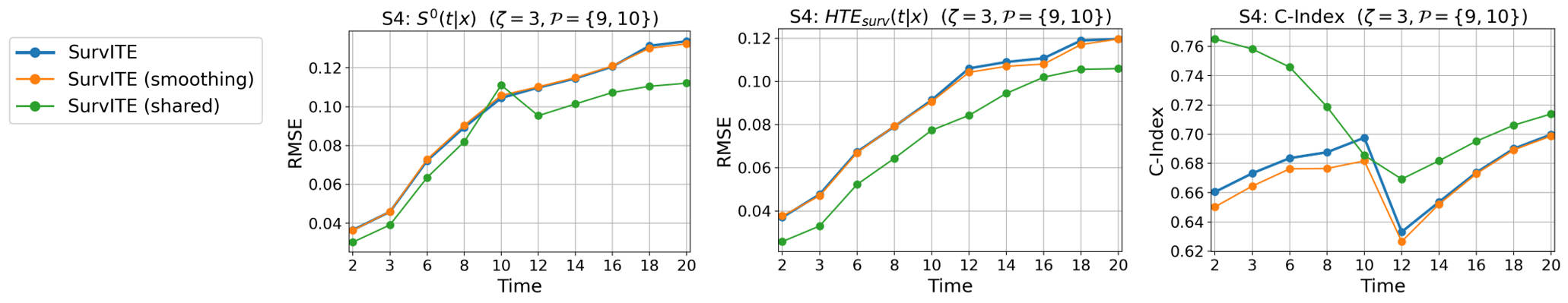}
		\caption{Training samples: 2000}
	\end{subfigure}~
    \caption{RMSE of estimating the treatment-specific survival function $S^{0}(t|x)$, that of the treatment effect $HTE_{surv}(t|x)$, and C-index of the survival predictions $S(t|x)$ using smoothing and parameter sharing on S4 with $\zeta=3$ and $\mathcal{P}=\{9,10\}$. Averaged across 5 runs.} \label{fig:parameters} 
\end{figure*}

\end{document}